\documentclass{article}

\PassOptionsToPackage{numbers, compress}{natbib}

\usepackage[preprint]{neurips_2021}

\usepackage[utf8]{inputenc} 
\usepackage[T1]{fontenc}    
\usepackage{hyperref}       
\usepackage{url}            
\usepackage{booktabs}       
\usepackage{amsfonts}       
\usepackage{nicefrac}       
\usepackage{microtype}      
\usepackage{xcolor}         

\usepackage{subfigure}
\usepackage{booktabs}
\usepackage{amsthm}
\usepackage{amsmath}
\usepackage{amssymb}
\usepackage{mathtools}
\usepackage{bm}
\usepackage{tabularx}
\usepackage{enumitem}
\usepackage{array}
\usepackage{cancel}
\usepackage{multirow}
\usepackage{wrapfig}

\theoremstyle{remark}
\newtheorem{remark}{Remark}

\theoremstyle{definition}
\newtheorem{definition}{Definition}
\theoremstyle{plain}
\newtheorem{proposition}{Proposition}

\newtheorem{lemma}{Lemma}

\newcolumntype{M}[1]{>{\centering\arraybackslash}m{#1}}

\newcommand{\Mid}{\;\|\;}
\newcommand{\tr}{\mathsf{T}}

\renewcommand{\P}{{\mathrm{P}}}
\newcommand{\A}{{\mathrm{A}}}

\usepackage{pgfplots}
\pgfplotsset{compat=1.16}
\usepackage{tikz}
\usetikzlibrary{backgrounds}
\usetikzlibrary{arrows}
\usetikzlibrary{shapes,shapes.geometric,shapes.misc}
\usetikzlibrary{bayesnet}
\usetikzlibrary{external}
\pgfplotsset{
    /pgfplots/xlabel near ticks/.style={
        /pgfplots/every axis x label/.style={
        at={(ticklabel cs:0.5)},anchor=near ticklabel
        }
    },
    /pgfplots/ylabel near ticks/.style={
        /pgfplots/every axis y label/.style={
        at={(ticklabel cs:0.5)},rotate=90,anchor=near ticklabel}
    }
}

\newcommand\ttiny{\fontsize{5pt}{5pt}\selectfont}
\pgfplotsset{every axis/.append style={
    label style={font=\scriptsize},
    tick label style={font=\ttiny},
}}

\definecolor{myred}{RGB}{255,31,91}
\definecolor{mygreen}{RGB}{0,205,108}
\definecolor{myblue}{RGB}{0,154,222}
\definecolor{mypurple}{RGB}{175,88,186}
\definecolor{myyellow}{RGB}{255,198,30}
\definecolor{myorange}{RGB}{242,133,34}
\definecolor{mygray}{RGB}{160,177,186}
\definecolor{mybrown}{RGB}{166,118,29}
\definecolor{myred2}{RGB}{233,0,45}
\definecolor{mygreen2}{RGB}{0,176,0}

\title{Physics-Integrated Variational Autoencoders for Robust and Interpretable Generative Modeling}

\author{%
  Naoya Takeishi, ~ Alexandros Kalousis \\
  University of Applied Sciences and Arts Western Switzerland (HES-SO) \\
  Geneva, Switzerland \\
  \texttt{\{naoya.takeishi,alexandros.kalousis\}@hesge.ch} \\
}

\begin{document}

\maketitle

\begin{abstract}
Integrating physics models within machine learning models holds considerable promise toward learning robust models with improved interpretability and abilities to extrapolate. In this work, we focus on the integration of incomplete physics models into deep generative models. In particular, we introduce an architecture of variational autoencoders (VAEs) in which a part of the latent space is grounded by physics. A key technical challenge is to strike a balance between the incomplete physics and trainable components such as neural networks for ensuring that the physics part is used in a meaningful manner. To this end, we propose a regularized learning method that controls the effect of the trainable components and preserves the semantics of the physics-based latent variables as intended. We not only demonstrate generative performance improvements over a set of synthetic and real-world datasets, but we also show that we learn robust models that can consistently extrapolate beyond the training distribution in a meaningful manner. Moreover, we show that we can control the generative process in an interpretable manner.
\end{abstract}

\section{Introduction}

Data-driven modeling is often opposed to theory-driven modeling, yet their integration has also been recognized as an important approach called \emph{gray-box} or \emph{hybrid} modeling.
In statistical machine learning, incorporation of mathematical models of physics (in a broad sense; including knowledge of biology, chemistry, economics, etc.) has also been attracting attention.
Gray-box / hybrid modeling in machine learning holds considerable promise toward learning robust models with improved abilities to extrapolate beyond the distributions that they have been exposed to during training.
Moreover, it can bring significant benefits in terms of model interpretability since parts of a model get semantically grounded to concrete domain knowledge.

A technical challenge in \emph{deep} gray-box modeling is to ensure an appropriate use of physics models. 
A careless design of models and learning can lead to an erratic behavior of the components meant to represent physics (e.g., with erroneous estimation of physics parameters), and eventually, the overall model just learns to ignore them.
This is particularly the case when we bring together simplified or imperfect physics models with highly expressive data-driven machine learning models such as deep neural networks.
Such cases call for principled methods for striking an appropriate balance between physics and data-driven models to prevent the detrimental effects during learning.

Integration of physics models into machine learning has been considered in various contexts (see, e.g., \citep{willardIntegratingPhysicsbasedModeling2020,vonruedenInformedMachineLearning2020} and our Section~\ref{related}), but most existing studies focus on prediction or forecasting tasks and are not directly applicable to other tasks.
More importantly, the careful orchestration of physics-based and data-driven components have not necessarily been considered.
A notable exception is \citet{leguenAugmentingPhysicalModels2021}, in which they proposed a method to regularize the action of trainable components of a hybrid model of differential equations.
Their method has been developed for dynamics forecasting with additive combinations of physics and trainable models, but application to other situations is not trivial.

In this work, we aim at the integration of incomplete physics models into deep generative models.
While we focus on variational autoencoders (VAEs, \citep{kingmaAutoencodingVariationalBayes2014,rezende_stochastic_2014}), our idea is applicable to other models in principle.
In our VAE, the decoder comprises physics-based models and trainable neural networks, and some of the latent variables are semantically grounded to the parameters of the physics models.
Such a VAE, if appropriately trained, is by construction partly interpretable.
Moreover, since it can by construction capture the underlying physics, it will be robust in out-of-distribution regime and exhibit meaningful extrapolation properties.
We propose a regularized learning framework for ensuring the meaningful use of the physics models and the preservation of the semantics of the latent variables in the physics-integrated VAEs.
We empirically demonstrate that our method can learn a model that exhibits better generalization, and more importantly, can extrapolate robustly in out-of-distribution regime.
In addition, we show how the direct access to the physics-grounded latent variables allows us to alter properties of generation meaningfully and explore counterfactual scenarios.

\section{Physics-integrated VAEs}
\label{model}

We first describe the structure of VAEs we consider, which comprise physics models and machine learning models such as neural nets.
We suppose that the physics models can be solved analytically or numerically with a reasonable cost, and the (approximate) solution is differentiable with regard to the quantities on which the solution depends.
This assumption holds in most physics models known in practice, which come in different forms such as algebraic and differential equations.
If there is no closed-form solution of algebraic equations, we can utilize differentiable optimizers \citep{amosOptNetDifferentiableOptimization2017} as a layer of the model.
For differential equations, differentiable integrators \citep[see, e.g.,][]{chenNeuralOrdinaryDifferential2018} will constitute a layer.
Handling non-differentiable and/or overly-complex simulators remains an important open challenge.


\subsection{Example}
\label{model:example}

We start with an example to demonstrate the main concepts.
Let us suppose that data comprise time-series of the angle of pendulums following an ordinary differential equation (ODE):
\begin{equation}\label{eq:pendulum}
    \underbrace{\mathrm{d}^2\vartheta(t)/\mathrm{d}t^2 + \omega^2 \sin\vartheta(t)}_{\text{given as prior knowledge,}~f_\P} + \underbrace{\xi \mathrm{d}\vartheta(t)/\mathrm{d}t - u(t)}_{\text{to be learned by NN,}~f_\A} = 0,
\end{equation}
where $\vartheta$ is a pendulum's angle, and $\omega$, $\xi$, and $u$ are the pendulum's angular velocity, damping coefficient, and external force, respectively.
We suppose that a data point $x$ is a sequence of $\vartheta(t)$, i.e., $x = [ \vartheta(0) \ \vartheta(\Delta t) \ \cdots \ \vartheta((\tau-1)\Delta t) ]^\tr \in \mathbb{R}^\tau$ for some $\Delta t \in \mathbb{R}$ and $\tau \in \mathbb{N}$, where $\vartheta(t)$ denotes the solution of \eqref{eq:pendulum} with a particular configuration of $\omega$, $\xi$, and $u$.
In this example, we learn a VAE on a dataset comprising such $x$'s with different configurations of $\omega$, $\xi$, and $u$.

Suppose that the first two terms of \eqref{eq:pendulum} are given as prior knowledge, i.e., we know that the governing equation should contain $f_\P(\vartheta,z_\P) \coloneqq \ddot\vartheta+z_\P^2\sin\vartheta$.
We will use such prior knowledge, $f_\P$, by incorporating it in the decoder of a VAE that we will learn.
Since $f_\P$ misses some effects of the true system \eqref{eq:pendulum}, we complete it by augmenting the decoder with a neural network $f_\A(\vartheta,\bm{z}_\A)$.
The VAE's latent variable will have two parts, $z_\P$ and $\bm{z}_\A$, respectively linked to $f_\P$ and $f_\A$.
On one hand, $\bm{z}_\A$ works as an ordinary VAE's latent variable since $f_\A$ is a neural net, and we suppose $\bm{z}_\A\in\mathbb{R}^d$, $p(\bm{z}_\A) \coloneqq \mathcal{N}(\bm{0},\bm{I})$.
On the other hand, we semantically ground $z_\P$ to a physics parameter; in this case, $z_\P\in\mathbb{R}$ should work as pendulum's $\omega$.
In summary, the augmented decoder here is $\mathbb{E}[x] = \operatorname{ODEsolve}_\vartheta \big[ f_\P(\vartheta(t), z_\P) + f_\A(\vartheta(t), \bm{z}_\A) = 0 \big]$, where $\operatorname{ODEsolve}_\vartheta$ denotes some differentiable solver of an ODE with regard to $\vartheta$.
The encoder will have corresponding recognition networks for $z_\P$ and $\bm{z}_\A$.
The situation in this example will be numerically examined in Section~\ref{expts:pendulum}.


\subsection{General formulation}
\label{model:general}


We now present the concept of our physics-integrated VAEs in a general form.
Note that our interest is not limited to the additive model combination nor ODEs.
In fact, the general formulation below subsumes non-additive augmentation of various physics models.
The notation introduced in this section will be used to explain the proposed regularized learning method later in Section~\ref{reg}.

For ease of discussion, we suppose that a VAE decoder comprises two parts: a physics-based model $f_\P$ and a trainable auxiliary function $f_\A$.
More general cases, for example with multiple trainable functions $f_{\A,1}, f_{\A,2}, \dots$ used in different ways, are handled in Appendix~A.


\subsubsection{Latent variables and priors}

We consider two types of latent variables, $\bm{z}_\P\in\mathcal{Z}_\P$ and $\bm{z}_\A\in\mathcal{Z}_\A$, which respectively will be used in $f_\P$ and $f_\A$.
The latent variables can be in any space, but for the sake of discussion, we suppose $\mathcal{Z}_\P$ and $\mathcal{Z}_\A$ are (subsets of) the Euclidean space and set their prior distribution as multivariate normal:
\begin{equation}\label{eq:prior}
    p(\bm{z}_\P) \coloneqq \mathcal{N}(\bm{z}_\P \mid \bm{m}_\P, v^2_\P \bm{I})
    \quad\text{and}\quad
    p(\bm{z}_\A) \coloneqq \mathcal{N}(\bm{z}_\A \mid \bm{0}, \bm{I}),
\end{equation}
where $\bm{m}_\P$ and $v^2_\P$ are defined in accordance with prior knowledge of $f_\P$'s parameters.
Note that $\bm{z}_\P$ will be directly interpretable as they will be semantically grounded to the parameters of the physics model $f_\P$; for example in Section~\ref{model:example}, $z_\P \coloneqq \omega$ was the angular velocity of a pendulum.


\subsubsection{Decoder}

The decoder of a physics-integrated VAE comprises two types of functions\footnote{The distinction between $f_\P$ and $f_\A$ depends on the origin of the functional forms (and not if trainable or not).
The form of $f_\P$ depends on physics' insight and thus fixed.
On the other hand, the form of $f_\A$ is determined only from utility as a function appoximator, and we can use whatever useful (e.g., feed-forward NNs, RNNs, etc.).}, $f_\P \colon \mathcal{Z}_\P \to \mathcal{Y}_\P$ and $f_\A \colon \mathcal{Y}_\P \times \mathcal{Z}_\A \to \mathcal{Y}_\A$.
For notational convenience, we consider a functional $\mathcal{F}$ that evaluates $f_\P$ and $f_\A$, solves an equation if any, and finally gives observation $\bm{x} \in \mathcal{X}$.
$\mathcal{X}$ may be the space of sequences, images, and so on.
Assuming Gaussian observation noise, we write the observation model as
\begin{equation}\label{eq:dec}
    p_\theta(\bm{x} \mid \bm{z}_\P, \bm{z}_\A) \coloneqq \mathcal{N} \big( \bm{x} \mid \mathcal{F} [ f_\A, f_\P; \bm{z}_\P, \bm{z}_\A ], \bm\Sigma_x \big),
\end{equation}
where $\bm{z}_A\in\mathcal{Z}_\A$ and $\bm{z}_\P\in\mathcal{Z}_\P$ are the arguments of $f_\A$ and $f_\P$, respectively.
Note that $f_\A$ and $f_\P$ may have other arguments besides $\bm{z}_\A$ and $\bm{z}_\P$, respectively, but they are omitted for simplicity.
We denote the set of trainable parameters of $f_\A$ and $f_\P$ (and $\bm\Sigma_x$) by $\theta$, while $f_\P$ may have no trainable global parameters other than $\bm{z}_\P$.

Let us see the semantics of the functional\footnote{It is natural to consider that $\mathcal{F}$ is a functional (and not a function) because we may need the access to the functions $f_\A$ and $f_\P$ themselves, rather than their pointwise values. For example, we need the full access to those functions when the decoder has an ODE solver with arbitrary initial condition.} $\mathcal{F}$ first in the light of the example of Section~\ref{model:example}.
Recall that there we considered the additive augmentation of ODE (as in \citep{leguenAugmentingPhysicalModels2021} and other studies).
It is subsumed by the expression \eqref{eq:dec} by setting $\mathcal{F} [ f_\A, f_\P; \bm{z}_\P, \bm{z}_\A ] \coloneqq \operatorname{ODEsolve}[f_\P(\bm{z}_\P) + f_\A(\bm{z}_\A) = 0]$.
Let us generalize the idea.
Our definition of the decoder in \eqref{eq:dec} allows not only additive augmentation of ODE but also broader range of architectures.
The composition of $f_\P$ and $f_\A$ is not limited to be additive because we consider general composition of functions $f_\A$ and $f_\P$.
Moreover, the form of the physics model is not limited to ODEs.
We list some examples of the configuration:
\begin{itemize}[leftmargin=*,itemsep=0pt,topsep=0pt]
    \item If equation $f_\P=0$ has a closed-form solution $S_{f_\P}\in\mathcal{Y}_\P$ (assuming that the solution space coincides with $\mathcal{Y}_\P$, just for ease of discussion), then $\mathcal{F}$ is simply an evaluation of $f_\A$, for example, $\mathcal{F}[f_\P,f_\A;\bm{z}_\A] \coloneqq f_\A(S_{f_\P}, \bm{z}_\A)$.
    \item If an algebraic equation $f_\P=0$ or $f_\A \circ f_\P=0$ has no closed-form solution, then $\mathcal{F}$ will have a differentiable optimizer, e.g., $\mathcal{F}[f_\P,f_\A] \coloneqq f_\A(\arg\min \Vert f_\P \Vert^2)$ or $\mathcal{F} \coloneqq \arg\min \Vert f_\A \circ f_\P \Vert^2$.
    \item $f_\P=0$ or $f_\A \circ f_\P=0$ can be a stochastic differential equation (and $\mathcal{F}$ contains its solver), for which $\bm{z}_\P$ and/or $\bm{z}_\A$ would become a sequence encoding the realization of the process noise.
\end{itemize}
The role of $f_\A$ can also be diverse; it can work not only as a complement of physics models inside equations, but also as correction of numerical errors of solvers or optimizers, downsampling or upsampling, and observables (e.g., from angle sequence to video of a pendulum).


\subsubsection{Encoder}

The encoder of a physics-integrated VAE accordingly comprises two parts: for posterior inference of $\bm{z}_\P$ and for that of $\bm{z}_\A$.
We consider the following decomposition of the approximated posterior:
\begin{equation}\label{eq:enc}\begin{gathered}
    q_\psi(\bm{z}_\P, \bm{z}_\A \mid \bm{x})
    \coloneqq q_\psi (\bm{z}_\A \mid \bm{x}) q_\psi (\bm{z}_\P \mid \bm{x}, \bm{z}_\A),
    \\
    \text{where}\quad
    q_\psi (\bm{z}_\A \mid \bm{x})
    \coloneqq \mathcal{N} \big( \bm{z}_\A \mid g_\A(\bm{x}), \bm\Sigma_\A \big),
    \quad
    q_\psi (\bm{z}_\P \mid \bm{x}, \bm{z}_\A)
    \coloneqq \mathcal{N} \big( \bm{z}_\P \mid g_\P(\bm{x}, \bm{z}_\A), \bm\Sigma_\P \big).
\end{gathered}\end{equation}
$g_\A \colon \mathcal{X} \to \mathcal{Z}_\A$ and $g_\P \colon \mathcal{X} \times \mathcal{Z}_\A \to \mathcal{Z}_\P$ are recognition networks.
We denote the trainable parameters of $g_\A$ and $g_\P$ (and $\bm\Sigma_\A$ and $\bm\Sigma_\P$) as $\psi$.
This particular dependency is for our regularization method in Section~\ref{reg:dataug}, where $g_\P$ should first remove the information of $\bm{z}_\A$ from $\bm{x}$ and then infer $\bm{z}_\P$.


\subsection{Evidence lower bound}

The VAE is to be learned as usual by maximizing the lower bound of the marginal log likelihood known as evidence lower bound (ELBO).
In our case, it is straightforward to derive:
\begin{equation} \label{eq:elbo}\begin{aligned}
    &
    \text{ELBO}(\theta,\psi; \bm{x}) =
    \mathbb{E}_{q_\psi(\bm{z}_\P, \bm{z}_\A \mid \bm{x})} \log p_\theta(\bm{x} \mid \bm{z}_\P, \bm{z}_\A)
    \\[-2pt]
    &\qquad\qquad\qquad\quad ~
    - D_\mathrm{KL} \big[ q_\psi (\bm{z}_\A \mid \bm{x}) \Mid p(\bm{z}_\A) \big]
    - \mathbb{E}_{q_\psi (\bm{z}_\A \mid \bm{x})} D_\mathrm{KL} \big[ q_\psi (\bm{z}_\P \mid \bm{x}, \bm{z}_\A) \Mid p(\bm{z}_\P) \big].
\end{aligned}\end{equation}
\section{Striking balance between physics and trainable models}
\label{reg}

We propose a regularized learning objective for physics-integrated VAEs.
It comprises two types of regularizers.
The first is for regularizing unnecessary flexibility of function approximators like neural networks and presented in Section~\ref{reg:harness}.
The second is for grounding encoder's output to physics parameters and presented in Section~\ref{reg:dataug}.
The overall objective is summarized in Section~\ref{reg:overall}.


\subsection{Regularizing excess flexibility of trainable functions}
\label{reg:harness}

If the trainable component of the physics-integrated VAE (i.e., $f_\A$) has rich expression capability, as is often the case with deep neural networks, merely maximizing the ELBO in \eqref{eq:elbo} provides no guarantee that the physics-based component (i.e., $f_\P$) will be used in a meaningful manner; e.g., $f_\P$ may just be ignored.
We want to ensure that $f_\A$ does not unnecessarily dominate the behavior of the entire model and that $f_\P$ is not ignored.
To this end, we borrow an idea from the \emph{posterior predictive check} (PPC), a procedure to check the validity of a statistical model \citep[see, e.g.,][]{BDA3}.
Whereas the standard PPCs examine the discrepancy between distributions of a model and data, we compute the discrepancy between those of the model and its ``physics-only'' reduced version, for monitoring and balancing the contributions of parts of the model.

For the sake of argument, suppose that a given physics model $f_\P$ is completely correct for given data.
Then, the discrepancy between the original model and its ``physics-only'' reduced model (where $f_\A$ is somehow invalidated) should be close to zero because the decoder of both the original model (with $f_\P$ and $f_\A$ working) and the reduced model (with only $f_\P$ working) should coincide in an ideal limit with the true data-generating process.
Even if $f_\P$ captures only a part of the truth, the discrepancy should be kept small, if not zero, to ensure meaningful use of the physics models in the overall model.

The ``physics-only'' reduced model is created as follows.
Recall that the original VAE is defined by Eqs.~\eqref{eq:dec} and \eqref{eq:enc}.
We define the decoder of the reduced model by replacing $f_\A \colon \mathcal{Y}_\P \times \mathcal{Z}_\A \to \mathcal{Y}_\A$ of \eqref{eq:dec} with a \emph{baseline function} $h_\A \colon \mathcal{Y}_\P \to \mathcal{Y}_\A$.
That is, the reduced observation model is
\begin{equation}\label{eq:dec_P}
    p_{\theta^\mathrm{r}}^\mathrm{r} (\bm{x} \mid \bm{z}_\P, \bm{z}_\A) \coloneqq \mathcal{N} \big( \bm{x} \mid \mathcal{F}[ h_\A, f_\P; \bm{z}_\P ], \bm\Sigma_x \big),
    \tag{\ref*{eq:dec}r}
\end{equation}
where we omit $\bm{z}_\A$ from the argument of $\mathcal{F}$ because $h_\A$ no longer takes it.
We denote the set of the trainable parameters of such a model as $\theta^\mathrm{r} \coloneqq \theta \backslash \operatorname{param}(f_\A) \cup \operatorname{param}(h_\A)$.
The corresponding encoder is defined as follows.
Recall that in the original model, posterior distributions of both $\bm{z}_\P$ and $\bm{z}_\A$ are inferred in \eqref{eq:enc} and then used for reconstructing each input $\bm{x}$ in \eqref{eq:dec}.
On the other hand, in the ``physics-only'' reduced model, $\bm{z}_\A$ is not referred to by \eqref{eq:dec_P}, which makes it less meaningful to place a particular posterior of $\bm{z}_\A$ for each $\bm{x}$.
Hence, we define the ``physics-only'' encoder by marginalizing out $\bm{z}_\A$ and using prior\footnote{It is just for defining $q_\psi^\mathrm{r}$ on the common support with $q_\psi$. Any non-informative distributions of $\bm{z}_\A$ are fine.} $p(\bm{z}_\A)$ instead.
That is, the reduced posterior is
\begin{equation}\label{eq:enc_P}
    q_\psi^\mathrm{r} (\bm{z}_\A, \bm{z}_\P \mid \bm{x}) \coloneqq
    p(\bm{z}_\A)
    \int q_\psi (\bm{z}_\P, \bm{z}_\A \mid \bm{x}) \mathrm{d}\bm{z}_\A.
    \tag{\ref*{eq:enc}r}
\end{equation}

%

Below we give a guideline for the choice of the baseline function, $h_\A$:
\begin{itemize}[leftmargin=*,itemsep=0pt,topsep=0pt]
    \item If the ranges of $f_\P$ and $f_\A$ are the same (i.e., $\mathcal{Y}_\P=\mathcal{Y}_\A$), then $h_\A$ can be an identity function $h_\A=\mathrm{Id}$.
    Note that in the additive case $f_\A \circ f_\P = f_\P + f_{\A'}$, where $f_{\A'}$ is a trainable function, replacing $f_\A$ with $h_\A=\mathrm{Id}$ is equivalent to replacing $f_{\A'}$ with $h_{\A'}=0$.
    \item If $\mathcal{Y}_\P \neq \mathcal{Y}_\A$, then $h_\A$ can be a linear or affine map from $\mathcal{Y}_\P$ to $\mathcal{Y}_\A$.
    For example, if $\mathcal{Y}_\P=\mathbb{R}^{d_\P}$ and $\mathcal{Y}_\A=\mathbb{R}^{d_\A}$ ($d_\P \neq d_\A$), then we can set $h_\A(f_\P(\bm{z}_\P)) = \bm{W} f_\P(\bm{z}_\P)$ where $\bm{W}\in\mathbb{R}^{d_\A \times d_\P}$.
\end{itemize}

The idea is to minimize the discrepancy between the full model and the ``physics-only'' reduced model.
In particular, we minimize the discrepancy between the posterior predictive distributions
\begin{equation}\label{eq:lact_org}\begin{gathered}
    D_\mathrm{KL} \big[ p_{\theta,\psi}(\tilde{\bm{x}} \mid X) \Mid p_{\theta^\mathrm{r},\psi}^\mathrm{r}(\tilde{\bm{x}} \mid X) \big],
    \quad\text{where}
    \\
    p_{\theta,\psi} (\tilde{\bm{x}} \mid X)
    = \int p_\theta (\tilde{\bm{x}} \mid \bm{z}_\P, \bm{z}_\A)
        q_\psi (\bm{z}_\P, \bm{z}_\A \mid \bm{x})
        p_\mathrm{d}(\bm{x} \mid X)
        \mathrm{d}\bm{z}_\P \mathrm{d}\bm{z}_\A \mathrm{d}\bm{x},
    \\
    p_{\theta^\mathrm{r},\psi}^\mathrm{r} (\tilde{\bm{x}} \mid X)
    = \int p_{\theta^\mathrm{r}}^\mathrm{r} (\tilde{\bm{x}} \mid \bm{z}_\P, \bm{z}_\A)
        q_\psi^\mathrm{r} (\bm{z}_\P, \bm{z}_\A \mid \bm{x})
        p_\mathrm{d}(\bm{x} \mid X)
        \mathrm{d}\bm{z}_\P \mathrm{d}\bm{z}_\A \mathrm{d}\bm{x}.
\end{gathered}\end{equation}
$p_\mathrm{d}(\bm{x} \mid X)$ is the empirical distribution with the support on data $X \coloneqq \{ \bm{x}_1, \dots, \bm{x}_n \}$.
We use $\tilde{\bm{x}}$, instead of $\bm{x}$, just for avoiding notational confusion by clarifying the target of integral $\int \mathrm{d}\bm{x}$.


Unfortunately, analytically computing \eqref{eq:lact_org} is usually intractable.
Hence, we take the following upper bound of \eqref{eq:lact_org} (a proof is in Appendix~B, and further remarks are in Appendix~C):
\begin{proposition}\label{prop:ub}
    Let $p_\theta$ and $p_\theta^\mathrm{r}$ be the shorthand of $p_\theta(\tilde{\bm{x}} \mid \bm{z}_\P, \bm{z}_\A)$ in \eqref{eq:dec} and $p_{\theta^\mathrm{r}}^\mathrm{r}(\tilde{\bm{x}} \mid \bm{z}_\P, \bm{z}_\A)$ in \eqref{eq:dec_P}, respectively.
    Let $p_\P$ and $p_\A$ be some distributions of $\bm{z}_\P$ and $\bm{z}_\A$, e.g., $p(\bm{z}_\P)$ and $p(\bm{z}_\A)$ using the priors in \eqref{eq:prior}, respectively.
    The KL divergence in \eqref{eq:lact_org} can be upper bounded as follows:
    \begin{multline}
        D_\mathrm{KL} \big[ p_{\theta,\psi}(\tilde{\bm{x}} \mid X) \Mid p_{\theta^\mathrm{r},\psi}^\mathrm{r}(\tilde{\bm{x}} \mid X) \big]
        \leq
        \mathbb{E}_{p_\mathrm{d}(\bm{x} \mid X)} \Big[
            \mathbb{E}_{q_\psi(\bm{z}_\P, \bm{z}_\A \mid \bm{x})} D_\mathrm{KL} [ p_\theta \Mid p_\theta^\mathrm{r} ]
            \\
            + D_\mathrm{KL} [ q_\psi(\bm{z}_\A \mid \bm{x}) \Mid p_\A ] + \mathbb{E}_{q_\psi(\bm{z}_\A \mid \bm{x})} D_\mathrm{KL} [ q_\psi(\bm{z}_\P \mid \bm{z}_\A, \bm{x}) \Mid p_\P ]
        \Big].
        \label{eq:lact_ub}
    \end{multline}
\end{proposition}
\begin{definition}
    Let us denote the upper bound \eqref{eq:lact_ub} by $\mathbb{E}_{p_\mathrm{d}(\bm{x} \mid X)}\hat{D}(\theta,\operatorname{param}(h),\psi;\bm{x})$.
    The regularization for inhibiting unnecessary flexibility of trainable functions is defined as minimization of
    \begin{equation}\label{eq:pcc}
        R_\text{PPC}(\theta,\operatorname{param}(h),\psi) \coloneqq \mathbb{E}_{p_\mathrm{d}(\bm{x} \mid X)}\hat{D}(\theta,\operatorname{param}(h),\psi;\bm{x}).
    \end{equation}
\end{definition}
\begin{remark}
    When multiple trainable functions are differently used in a model (e.g., inside \emph{and} outside an equation solver), which is often the case in practice, the definition of $R_\text{PPC}$ should be generalized to consider marginal contribution of every trainable function.
    See Appendix~A.
\end{remark}



\subsection{Grounding physics encoder by physics-based data augmentation}
\label{reg:dataug}

Toward properly learning physics-integrated VAEs, minimizing $R_\text{PPC}$ solely may not be enough because inferred $\bm{z}_\P$ may be still meaningless but makes $R_\text{PPC}$ not that large (e.g., with solution of $f_\P$ fluctuating around the mean pattern of data), and then optimization may not be able to escape such local minima.
Though it is difficult to avoid such a local solution perfectly, we can alleviate the situation by considering additional objectives to encourage a proper use of the physics.

The idea is to use the physics model as a source of information for data augmentation, which helps us to ground the output of the recognition network, $g_\P$ in \eqref{eq:enc}, to the parameters of $f_\P$.
We want to draw some $\bm{z}_\P$, feed it to the physics model $f_\P$ (and a solver if any), and use the generated signal as additional data during training.
A technical challenge to this end is that because the physics model may be incomplete, the artificial signals from it and the real signals may have different natures.
To compensate such difference, we arrange a particular functionality of the physics encoder, $g_\P$.
\begin{wrapfigure}[13]{r}{37mm}
    \centering\vspace*{-8pt}
        \scalebox{0.68}{\begin{tikzpicture}[inner sep=0pt, outer sep=0pt]
    \node[latent, label={[yshift=1mm] $\bm{x}$}] (x) {};
    \node[factor, right=5mm of x, yshift=5mm, label={[yshift=1mm] $g_\A$}, fill=mygreen] (ga) {};
    \node[factor, right=5mm of x, yshift=-5mm, label={[yshift=1mm,xshift=1mm] $g_{\P,1}$}, fill=myyellow] (gp1) {};
    \node[latent, right=5mm of gp1] (gp1val) {};
    \node[factor, right=5mm of gp1val, label={[yshift=1mm] $g_{\P,2}$}, fill=myyellow] (gp2) {};
    \node[latent, right=5mm of gp2, label={[yshift=1mm] $\bm{z}_\P$}] (zp) {};
    \node[latent, right=21mm of ga, label={[yshift=1mm] $\bm{z}_\A$}] (za) {};
    \node[factor, right=5mm of za, label={[yshift=1mm] $f_\A$}, fill=myblue] (fa) {};
    \node[factor, right=5mm of zp, label={[yshift=1mm] $f_\P$}, fill=myred] (fp) {};
    \node[latent, right=43mm of x, label={[yshift=1mm] $\hat{\bm{x}}$}] (xhat) {};
    \node[factor, below=8mm of fp, label={[yshift=1mm] $h$}, fill=mypurple] (h) {};
    \node[latent, below=5mm of xhat, label={[yshift=1mm] $\bm{x}^\mathrm{r}$}] (xr) {};
    \node[factor, below=7mm of zp, xshift=-4mm, label={[yshift=1mm] $R_{\text{DA},1}$}, fill=mygray] (rda1) {};
    \edge[dashed, bend right=75] {gp1val} {rda1};
    \edge[dashed, bend left=75] {xr} {rda1};
    \edge[] {x} {ga};
    \edge[] {x} {gp1};
    \edge[] {gp1} {gp1val};
    \edge[] {gp1val} {gp2};
    \edge[] {gp2} {zp};
    \edge[] {ga} {za};
    \edge[] {za} {fa};
    \edge[] {zp} {fp};
    \edge[] {fa} {xhat};
    \edge[] {fp} {xhat};
    \edge[] {fp} {xr};
    \edge[] {h} {xr};
\end{tikzpicture}}
        \par
        \vspace*{-8pt}
        \scalebox{0.68}{\begin{tikzpicture}[inner sep=0pt, outer sep=0pt]
    \node[latent, label={[yshift=1mm] $\bm{z}_\P^\star$}] (zpstar) {};
    \node[factor, right=5mm of zpstar, label={[yshift=1mm] $f_\P$}, fill=myred] (fp) {};
    \node[factor, above=8mm of fp, label={[yshift=1mm] $h$}, fill=mypurple] (h) {};
    \node[latent, right=5mm of fp, yshift=5mm, label={[yshift=1mm] $\bm{x}^\mathrm{r}$}] (xr) {};
    \node[factor, right=20mm of xr, label={[yshift=1mm] $g_{\P,2}$}, fill=myyellow] (gp2) {};
    \node[latent, right=5mm of gp2, label={[yshift=1mm] $\bm{z}_\P$}] (zp) {};
    \node[factor, below=7mm of zp, xshift=-20mm, label={[yshift=1mm] $R_{\text{DA},2}$}, fill=mygray] (rda2) {};
    \edge[dashed, bend right=40] {zpstar} {rda2};
    \edge[dashed, bend left=40] {zp} {rda2};
    \edge[] {zpstar} {fp};
    \edge[] {h} {xr};
    \edge[] {fp} {xr};
    \edge[] {xr} {gp2};
    \edge[] {gp2} {zp};
\end{tikzpicture}}
    \vspace*{-11pt}
    \caption{Diagrams of (\emph{upper}) $R_{\text{DA},1}$ in \eqref{eq:unmix} and (\emph{lower}) $R_{\text{DA},2}$ in \eqref{eq:dataug}.}
    \label{fig:gms}
\end{wrapfigure}
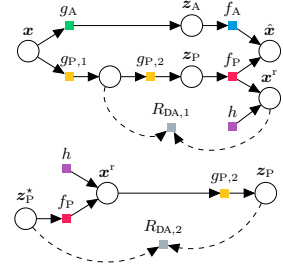

Let $\bm{z}^\star_\P$ be a sample drawn from some distribution of $\bm{z}_\P$ (e.g., prior $p(\bm{z}_\P)$).
We artificially generate signals $\bm{x}^\mathrm{r}(\bm{z}^\star_\P)$ by feeding $\bm{z}^\star_\P$ to the ``physics-only'' decoding process in \eqref{eq:dec_P}, that is,
\begin{equation}
    \bm{x}^\mathrm{r}(\bm{z}^\star_\P) \coloneqq \mathcal{F} [ h_\A, f_\P; \bm{z}_\P = \bm{z}^\star_\P ].
\end{equation}
We want the physics-part recognition network, $g_\P$, to successfully estimate $\bm{z}^\star_\P$ given the corresponding $\bm{x}^\mathrm{r}(\bm{z}^\star_\P)$, which is necessary to say that the result of the inference by $g_\P$ is grounded to the parameters of $f_\P$.
However, in general, real data $\bm{x}$ and the augmented data $\bm{x}^\mathrm{r}(\bm{z}^\star_\P)$ have different natures because $f_\P$ may miss some aspects of the true data-generating process.

We handle this issue by considering a specific design of the physics-part recognition network, $g_\P$.
We decompose $g_\P$ into two stages as $g_\P(\bm{x}, \bm{z}_\A) = g_{\P,2} (g_{\P,1}(\bm{x}, \bm{z}_\A))$ without loss of generality.
On one hand, $g_{\P,1}$ should transform real data $\bm{x}$ to signals that resemble the physics-based augmented signal, $\bm{x}^\mathrm{r}$.
In other words, $g_{\P,1}$ should ``cleanse'' real data into a virtual ``physics-only'' counterpart. 
We enforce such a functionality of $g_{\P,1}$ by making its output close to the following quantity:
\begin{equation}
    \bm{x}^\mathrm{r}(g_\P(\bm{x}, \bm{z}_\A)) = \mathcal{F} [ h_\A, f_\P; \bm{z}_\P = g_\P(\bm{x}, \bm{z}_\A) ].
\end{equation}
On the other hand, $g_{\P,2}$ should receive such ``cleansed'' input and return the (sufficient statistics of) posterior of $\bm{z}_\P$.
If the aforementioned functionality of $g_{\P,1}$ is successfully realized, we can directly self-supervise $g_{\P,2}$ with $\bm{x}^\mathrm{r}(\bm{z}^\star_\P)$ because $\bm{x}^\mathrm{r}(g_\P(\bm{x}, \bm{z}_\A))$ and $\bm{x}^\mathrm{r}(\bm{z}^\star_\P)$ should have similar nature.

In summary, we define a couple of regularizers for setting such functionality of $g_{\P,1}$ and $g_{\P,2}$ as follows (with the corresponding diagrams of computation shown in Figure~\ref{fig:gms}):
\begin{definition}
    Let $\operatorname{sg}[\cdot]$ be the stop-gradient operator.
    The regularization for the physics-based data augmentation is defined as minimization of
    \begin{align}
        R_{\text{DA},1}(\psi)
        &\coloneqq \mathbb{E}_{p_\mathrm{d}(\bm{x} \mid X) q(\bm{z}_\A \mid \bm{x})} \big\Vert
            g_{\P,1} (\bm{x}, \bm{z}_\A)
            - \operatorname{sg} \big[ \bm{x}^\mathrm{r}(g_\P(\bm{x}, \bm{z}_\A)) \big]
        \big\Vert_2^2
        \quad\text{and}
        \label{eq:unmix}
        \\
        R_{\text{DA},2}(\psi)
        &\coloneqq \mathbb{E}_{\bm{z}^\star_\P} \big\Vert g_{\P,2} \big( \operatorname{sg} \big[ \bm{x}^\mathrm{r}(\bm{z}^\star_\P) \big] \big) - \bm{z}^\star_\P \big\Vert_2^2.
        \label{eq:dataug}
    \end{align}
\end{definition}


\subsection{Overall regularized learning objective}
\label{reg:overall}

The overall regularized learning problem of the proposed physics-integrated VAEs is as follows:
\begin{equation*}
    \underset{\theta,\operatorname{param}(h),\psi}{\text{minimize}} ~~
    -\mathbb{E}_{p_\mathrm{d}(\bm{x} \mid X)}\text{ELBO}(\theta,\psi; \bm{x})
    + \alpha R_\text{PPC}(\theta,\operatorname{param}(h),\psi)
    + \beta R_{\text{DA},1}(\psi)
    + \gamma R_{\text{DA},2}(\psi),
\end{equation*}
where each term appears in \eqref{eq:elbo}, \eqref{eq:pcc}, \eqref{eq:unmix}, and \eqref{eq:dataug}, respectively.
Recall that $\theta$ and $\psi$ are the sets of the parameters of the full model's decoder \eqref{eq:dec} and encoder \eqref{eq:enc}, respectively, and that $\operatorname{param}(h)$ denotes the set of the parameters of $h$, which may be empty.
If we cannot specify a reasonable sampling distribution of $\bm{z}^\star_\P$ needed in \eqref{eq:dataug}, we do not use $R_{\text{DA},1}$ and $R_{\text{DA,2}}$; it may happen when the semantics of $\bm{z}_\P$ are not inherently grounded, e.g., when $f_\P$ is a \emph{neural} Hamilton's equation \citep{tothHamiltonianGenerativeNetworks2020}.

\section{Related work}
\label{related}

The integration of theory-driven and data-driven methodologies has been sought in various ways.
We overview some perspectives in this section and more in Appendix~D.

\paragraph{Physics+ML in model design}
Integration in model design, often called gray-box or hybrid modeling, has been studied for decades \citep[e.g.,][]{psichogiosHybridNeuralNetworkfirst1992,rico-martinezContinuoustimeNonlinearSignal1994,thompsonModelingChemicalProcesses1994} and is still active, with deep neural networks utilized in various areas \citep[e.g.,][]{youngPhysicallyBasedMachine2017,
raissiDeepHiddenPhysics2018,
longHybridNetIntegratingModelbased2018,
wanDataassistedReducedorderModeling2018,
nutkiewiczDatadrivenUrbanEnergy2018,
ajayAugmentingPhysicalSimulators2018,
ajayCombiningPhysicalSimulators2019,
debezenacDeepLearningPhysical2019,
zengTossingBotLearningThrow2019,
wangIntegratingModeldrivenDatadriven2019,
roehrlModelingSystemDynamics2020,
leguenDisentanglingPhysicalDynamics2020,
muralidharPhyNetPhysicsGuided2020,
belbute-peresCombiningDifferentiablePDE2020,
senguptaEnsemblingGeophysicalModels2020,
rackauckasUniversalDifferentialEquations2020,
liKohnShamEquationsRegularizer2020,
qian_integrating_2021,
silvestri_embedded-model_2021}.
Most recent studies focus on prediction, and the generative modeling has been less investigated.
Moreover, mechanisms to regularize the flexibility of trainable components have hardly been addressed.

The work of \citet{leguenAugmentingPhysicalModels2021} is notable here because they consider a mechanism to regularize the flexibility a trainable component to preserve the utility of physics in the model, even though it is only focused on dynamics learning for forecasting.
They learn an additive hybrid ODE model $\dot{x} = f_\P(x) + f_\A(x)$, where $f_\P$ is a prescribed physics model, and $f_\A$ is a neural network.
Such a model is subsumed in our architecture as exemplified in Section~\ref{model}.
Moreover, \citet{leguenAugmentingPhysicalModels2021} propose to regularize $f_\A$ by minimizing $\Vert f_\A \Vert^2$.
Such a term also appears in one of our regularizers, $R_\text{PPC}$; when the observation noise is Gaussian, the first term of the right-hand side of \eqref{eq:lact_ub} becomes $\mathbb{E} \Vert (f_\A \circ f_\P)-f_\P \Vert_2^2 = \mathbb{E} \Vert f_\P+f_{\A'} - f_\P \Vert_2^2 = \mathbb{E} \Vert f_{\A'} \Vert_2^2$.
Therefore, we get a ``VAE variant'' of \citet{leguenAugmentingPhysicalModels2021} by switching off a part of $R_\text{PPC}$ and the other regularizers, $R_{\text{DA},1}$ and $R_{\text{DA,2}}$.
We examine cases similar to it in our experiment for comparison.

\citet{yildizODE2VAEDeepGenerative2019} and \citet{linialGenerativeODEModeling2020} developed VAEs whose latent variable follows ODEs.
\citet{linialGenerativeODEModeling2020} also suggest grounding the semantics of the latent variable by providing sparse supervision on it.
It is feasible only when we have a chance to observe the latent variable (e.g., with an increased cost) and may often be inherently infeasible in some problem settings including ours.
In our method, we never assume availability of observation of latent variables and instead use the physics models in a self-supervised manner.
While direct comparison is not meaningful due to the difference of settings, we examine a baseline close to the base model of \citet{linialGenerativeODEModeling2020} in our experiment for comparison.

\citet{tothHamiltonianGenerativeNetworks2020} propose a model where the latent variable sequence is governed by the Hamiltonian mechanics with a neural Hamiltonian.
While it does not suppose very specific physics models but considers general mechanics, they can also be included in our framework; that is, $f_\P$ can be a Hamilton's equation with a neural Hamiltonian.
We try such a model in one of our experiments.

\paragraph{Physics+ML in objective design}
Another prevailing strategy is to define objective functions based on physics knowledge \citep[e.g.,][]{stewartLabelfreeSupervisionNeural2017,karpatnePhysicsguidedNeuralNetworks2017,raissiPhysicsinformedNeuralNetworks2019,jiaPhysicsGuidedRNNs2019,yangEnforcingDeterministicConstraints2019,kaltenbachIncorporatingPhysicalConstraints2020,zhangThermodynamicConsistentNeural2020,rixnerProbabilisticGenerativeModel2020,chenPhysicsinformedGenerativeAdversarial2020,wang_understanding_2021}.
In generative modeling, for example, \citet{stinisEnforcingConstraintsInterpolation2019} use residuals from physics models as a feature of GAN's discriminator.
\citet{golanySimGANsSimulatorbasedGenerative2020} regularize the generation from GANs by forcing it close to a prescribed physics relation.
These approaches are often easy to deploy, but an inherent limitation is that given physics knowledge should be complete to some extent, otherwise a physics-based loss is not well-defined.
\section{Experiments}
\label{expts}

We performed experiments on two synthetic datasets and two real-world datasets, for which we prepared instances of physics-integrated VAEs.
We show each particular architecture of physics-integrated VAEs and the corresponding results; some details are deferred to Appendix~E.
While direct comparison is impossible due to the differences of the problem settings, the baseline methods we examined (listed below) are similar to some existing methods \citep{aragon-calvoSelfsupervisedLearningPhysicsaware2020,yildizODE2VAEDeepGenerative2019,tothHamiltonianGenerativeNetworks2020,linialGenerativeODEModeling2020,leguenAugmentingPhysicalModels2021}.
\begin{description}[labelwidth=6em,labelindent=0em,leftmargin=!,topsep=0pt,itemsep=0pt,font={\ttfamily},labelsep*=1em]
    \item[NN-only] Ordinary VAE \citep{kingmaAutoencodingVariationalBayes2014,rezende_stochastic_2014}; the decoder is $\mathbb{E}\bm{x}=f_\A(\bm{z}_\A)$, where $f_\A$ is a neural net.
    \item[Phys-only] Physics VAE; the decoder is $\mathbb{E}\bm{x}=\mathcal{F}[f_\P;\bm{z}_\P]$ with no neural nets. The encoder is with neural nets as ordinary VAEs. This is almost equivalent to the method of \citet{aragon-calvoSelfsupervisedLearningPhysicsaware2020} when the problem is as in Section~\ref{expts:galaxy}.
    \item[NN+solver] VAE with physics solvers; the decoder is $\mathbb{E}\bm{x}=\mathcal{F}[f_\A;\bm{z}_\A]$, where $f_\A$ is a neural net, and $\mathcal{F}$ includes some equation-solving process (e.g., ODE/PDE solver), but no more physics-based knowledge is given (i.e., there is no $f_\P$). This is similar to the methods of, for example, \citet{yildizODE2VAEDeepGenerative2019} and \citet{tothHamiltonianGenerativeNetworks2020}. 
    \item[NN+phys] Physics-integrated VAE learned without the regularizers (i.e., $\alpha=\beta=\gamma=0$); this is similar to the base models of \citet{linialGenerativeODEModeling2020} and \citet{qian_integrating_2021}. Finer ablations are also studied, among which the cases with $\beta=0$ or $\gamma=0$ are similar to the model of \citet{leguenAugmentingPhysicalModels2021}.
    \item[NN+phys+reg] Our proposal; physics-integrated VAE learned with the proposed regularizers.
\end{description}

We aligned the total dimensionality of the latent variables of each method (except \texttt{phys-only}); when $\dim\bm{z}_\A=d_\A$ and $\dim\bm{z}_\P=d_\P$ in \texttt{NN+phys(+reg)}, we set $\dim\bm{z}_\A=d_\A+d_\P$ in \texttt{NN-only} and \texttt{NN+solver}.
The hyperparameters, $\alpha$, $\beta$, and $\gamma$, were chosen with validation set performance.
We investigated the performance sensitivity to them; no large degradation of performance was observed even if we changed the values by $\times 10$ or $\times \frac{1}{10}$ from the chosen values; details are in Appendix~F.


\subsection{Forced damped pendulum}
\label{expts:pendulum}

\begin{figure}
    \begin{minipage}[t]{\textwidth}
        \centering\vspace*{0pt}
        \begin{minipage}[t]{0.62\textwidth}
    \vspace*{2mm}\centering
    \pgfplotsset{height=3.1cm,width=8.9cm}
    \begin{tikzpicture}
        \begin{axis}[compat=newest,
            label style={font=\scriptsize}, xlabel={time $t$}, ylabel={$\vartheta(t)$},
            enlarge x limits=false,
            xmin=0, xmax=8, ymin=-2.5, ymax=1.9,
            xtick pos=left, ytick pos=left,
            legend entries={Truth, Phys-only, NN+solver, NN+phys+reg},
            legend style={at={(0.01,0.01)}, anchor=south west, nodes={scale=0.6, transform shape}, inner sep=1pt},
            legend columns=4,
            legend image code/.code={
                \draw[] 
                plot coordinates {
                    (0cm,0cm)
                    (0.12cm,0cm)        
                    (0.25cm,0cm)         
                };%
            },
            xlabel style={at={(0.5,-1ex)}},
            ]
            \addplot [black, thick] table [x=t, y=data] {pendulum/extrapolation_idx66.txt};
            \addplot [myyellow, dashed, thick] table [x=t, y=noaux] {pendulum/extrapolation_idx66.txt};
            \addplot [mygreen, densely dotted, very thick] table [x=t, y=nophy1] {pendulum/extrapolation_idx66.txt};
            \addplot [myred] table [x=t, y=phy_reg] {pendulum/extrapolation_idx66.txt};
            \addplot [mygray, ultra thick, densely dashed, opacity=0.8, on layer=background] coordinates {(2.45, -2.5) (2.45, 1.9)};
            \node[inner sep=0,outer sep=0,anchor=north west] at (2.55,1.8) {{\scriptsize extrapolation $\rightarrow$}};
            \node[inner sep=0,outer sep=0,anchor=north east] at (2.35,1.8) {{\scriptsize $\leftarrow$ reconstruction}};
        \end{axis}
    \end{tikzpicture}
\end{minipage}
\hfill
\begin{minipage}[t]{0.35\textwidth}
    \vspace*{0pt}\centering
    \caption{Reconstruction and extrapolation of a test sample of the pendulum data. Range $0 \leq t <2.5$ is reconstruction, whereas $t\geq2.5$ is extrapolation.}
    \label{fig:pendulum_extp}
\end{minipage}
    \end{minipage}
    \\
    \begin{minipage}[t]{\textwidth}
        \centering\vspace*{0pt}
        \pgfplotsset{height=2.9cm,width=3.1cm}
%
%
\begin{tikzpicture}
    \begin{axis}[compat=newest,
        label style={font=\scriptsize},
        enlarge x limits=false,
        ylabel=$\vartheta(t)$, 
        ymin=-2, ymax=2,
        xtick pos=left, ytick pos=left,
        legend entries={Truth},
        legend style={at={(0.04,0.04)}, anchor=south west, nodes={scale=0.5, transform shape}, inner sep=0pt},
        legend image code/.code={
            \draw[] 
            plot coordinates {
                (0cm,0cm)
                (0.15cm,0cm)        
                (0.3cm,0cm)         
            };%
        },
        tick label style={font=\ttiny},
        ]
        \addplot [black, thick] table [x=t, y=true] {pendulum/counterfactual_idx7_c0.4_omg0.86.txt};
        \addplot [myblue, very thick, dash dot] table [x=t, y=PAnr] {pendulum/counterfactual_idx7_c0.4_omg0.86.txt};
        \addplot [myred] table [x=t, y=PA] {pendulum/counterfactual_idx7_c0.4_omg0.86.txt};
        \node[inner sep=0,outer sep=0,fill=white,fill opacity=1.0,anchor=south west] at (0.12,1.30) {{\scriptsize $\omega\!=\!0.86$}};
    \end{axis}
\end{tikzpicture}
\begin{tikzpicture}
    \begin{axis}[compat=newest,
        label style={font=\scriptsize},
        enlarge x limits=false,
        yticklabels={,,}, 
        ymin=-2, ymax=2,
        xtick pos=left, ytick pos=left,
        legend entries={NN+phys},
        legend style={at={(0.04,0.04)}, anchor=south west, nodes={scale=0.5, transform shape}, inner sep=0pt},
        legend image code/.code={
            \draw[] 
            plot coordinates {
                (0cm,0cm)
                (0.15cm,0cm)        
                (0.3cm,0cm)         
            };%
        },
        tick label style={font=\ttiny},
        ]
        \addplot [myblue, very thick, dash dot] table [x=t, y=PAnr] {pendulum/counterfactual_idx7_c0.6_omg1.29.txt};
        \addplot [myred] table [x=t, y=PA] {pendulum/counterfactual_idx7_c0.6_omg1.29.txt};
        \addplot [black, thick] table [x=t, y=true] {pendulum/counterfactual_idx7_c0.6_omg1.29.txt};
        \node[inner sep=0,outer sep=0,fill=white,fill opacity=1.0,anchor=south west] at (0.12,1.30) {{\scriptsize $\omega\!=\!1.29$}};
    \end{axis}
\end{tikzpicture}
\begin{tikzpicture}
    \begin{axis}[compat=newest,
        label style={font=\scriptsize},
        enlarge x limits=false,
        yticklabels={,,}, 
        ymin=-2, ymax=2,
        xtick pos=left, ytick pos=left,
        legend entries={NN+phys+reg},
        legend style={at={(0.04,0.04)}, anchor=south west, nodes={scale=0.5, transform shape}, inner sep=0pt},
        legend image code/.code={
            \draw[] 
            plot coordinates {
                (0cm,0cm)
                (0.15cm,0cm)        
                (0.3cm,0cm)         
            };%
        },
        tick label style={font=\ttiny},
        ]
        \addplot [myred] table [x=t, y=PA] {pendulum/counterfactual_idx7_c0.8_omg1.72.txt};
        \addplot [black, thick] table [x=t, y=true] {pendulum/counterfactual_idx7_c0.8_omg1.72.txt};
        \addplot [myblue, very thick, dash dot] table [x=t, y=PAnr] {pendulum/counterfactual_idx7_c0.8_omg1.72.txt};
        \node[inner sep=0,outer sep=0,fill=white,fill opacity=1.0,anchor=south west] at (0.12,1.30) {{\scriptsize $\omega\!=\!1.72$}};
    \end{axis}
\end{tikzpicture}
\begin{tikzpicture}
    \begin{axis}[compat=newest,
        label style={font=\scriptsize},
        enlarge x limits=false,
        yticklabels={,,}, 
        ymin=-2, ymax=2,
        xtick pos=left, ytick pos=left,
        tick label style={font=\ttiny},
        ]
        \addplot [black, thick] table [x=t, y=true] {pendulum/counterfactual_idx7_c1.0_omg2.15.txt};
        \addplot [myred] table [x=t, y=PA] {pendulum/counterfactual_idx7_c1.0_omg2.15.txt};
        \addplot [myblue, very thick, dash dot] table [x=t, y=PAnr] {pendulum/counterfactual_idx7_c1.0_omg2.15.txt};
        \node[inner sep=0,outer sep=0,fill=white,fill opacity=1.0,anchor=south west] at (0.12,1.10) {{\scriptsize $\omega\!=\!\mathbb{E}[z_\P]$}};
        \node[inner sep=0,outer sep=0,fill=white,fill opacity=1.0,anchor=north west] at (0.8,0.9) {{\scriptsize $=\!2.15$}};
        \node[inner sep=0,outer sep=0,fill=white,fill opacity=1.0,anchor=south east] at (2.4,-1.9) {{\scriptsize original}};
    \end{axis}
\end{tikzpicture}
\begin{tikzpicture}
    \begin{axis}[compat=newest,
        label style={font=\scriptsize},
        enlarge x limits=false,
        yticklabels={,,}, 
        ymin=-2, ymax=2,
        xtick pos=left, ytick pos=left,
        tick label style={font=\ttiny},
        ]
        \addplot [black, thick] table [x=t, y=true] {pendulum/counterfactual_idx7_c1.5_omg3.22.txt};
        \addplot [myred] table [x=t, y=PA] {pendulum/counterfactual_idx7_c1.5_omg3.22.txt};
        \addplot [myblue, very thick, dash dot] table [x=t, y=PAnr] {pendulum/counterfactual_idx7_c1.5_omg3.22.txt};
        \node[inner sep=0,outer sep=0,fill=white,fill opacity=1.0,anchor=south west] at (0.12,1.30) {{\scriptsize $\omega\!=\!3.22$}};
    \end{axis}
\end{tikzpicture}
\begin{tikzpicture}
    \begin{axis}[compat=newest,
        label style={font=\scriptsize},
        enlarge x limits=false,
        yticklabels={,,}, 
        ymin=-2, ymax=2,
        xtick pos=left, ytick pos=left,
        tick label style={font=\ttiny},
        ]
        \addplot [black, thick] table [x=t, y=true] {pendulum/counterfactual_idx7_c2.0_omg4.29.txt};
        \addplot [myred] table [x=t, y=PA] {pendulum/counterfactual_idx7_c2.0_omg4.29.txt};
        \addplot [myblue, very thick, dash dot] table [x=t, y=PAnr] {pendulum/counterfactual_idx7_c2.0_omg4.29.txt};
        \node[inner sep=0,outer sep=0,fill=white,fill opacity=1.0,anchor=south west] at (0.12,1.30) {{\scriptsize $\omega\!=\!4.29$}};
    \end{axis}
\end{tikzpicture}
\begin{tikzpicture}
    \begin{axis}[compat=newest,
        label style={font=\scriptsize},
        enlarge x limits=false,
        yticklabels={,,}, 
        ymin=-2, ymax=2,
        xtick pos=left, ytick pos=left,
        tick label style={font=\ttiny},
        ]
        \addplot [black, thick] table [x=t, y=true] {pendulum/counterfactual_idx7_c2.5_omg5.36.txt};
        \addplot [myred] table [x=t, y=PA] {pendulum/counterfactual_idx7_c2.5_omg5.36.txt};
        \addplot [myblue, very thick, dash dot] table [x=t, y=PAnr] {pendulum/counterfactual_idx7_c2.5_omg5.36.txt};
        \node[inner sep=0,outer sep=0,fill=white,fill opacity=1.0,anchor=south west] at (0.12,1.30) {{\scriptsize $\omega\!=\!5.36$}};
    \end{axis}
\end{tikzpicture}
%
\\
\vspace*{-1ex}
\caption{Counterfactual generation for the pendulum data. Horizontal axis is time $t$. The center panel shows the original data, and the rest is the generation with $z_\P$ (i.e., $\omega$) altered while $\bm{z}_\A$ fixed.}
\label{fig:pendulum_cf}
    \end{minipage}
    \vspace*{-2ex}
\end{figure}

\paragraph{Dataset}
We generated data from \eqref{eq:pendulum} with $u(t)= A \omega^2 \cos(2 \pi \phi t)$.
Each data-point $\bm{x}$ is a sequence $\bm{x} \coloneqq [ \vartheta_1 \ \cdots \vartheta_\tau ]\in\mathbb{R}^\tau$, where $\vartheta_j$ is the value of a solution $\vartheta(t_j)$ at $t_j \coloneqq (j-1)\Delta t$.
We randomly drew a sample of the initial condition $\vartheta_1$ (with $\dot\vartheta_1=0$ fixed) and the values of $\omega$, $\zeta$, $A$, and $\phi$ for each sequence.
We generated 2,500 sequences of length $\tau=50$ with $\Delta t = 0.05$ and separated them into a training, validation, and test sets with 1,000, 500, and 1,000 sequences, respectively.

\paragraph{Setting}
We set $f_\P$ as in Section~\ref{model:example}, i.e., $f_\P(\vartheta,z_\P) \coloneqq \ddot\vartheta+z_\P^2\sin(\vartheta)$, where $z_\P\in\mathbb{R}$ should work as angular velocity $\omega$.
We augmented it by $f_{\A,1}(\vartheta, z_{\A,1})$ additively, where $f_{\A,1}$ was a multi-layer perceptron (MLP) and $z_{\A,1} \in \mathbb{R}$.
The ODE $f_\P+f_{\A,1}=0$ was solved with the Euler update scheme in the model.
The model had another MLP\footnote{\label{note:mlp}We used MLP as the data are fixed length. The same holds hereafter. Extension to other networks is easy.} $f_{\A,2}$ with another latent variable $\bm{z}_{\A,2} \in \mathbb{R}^2$ for further modifying the solution of the ODE.
In summary, the decoding process is $\mathcal{F} \coloneqq f_{\A,2}(\operatorname{solve}_\vartheta[f_\P(\vartheta,z_\P)+f_{\A,1}(\vartheta,z_{\A,1})=0], \bm{z}_{\A,2})$.
The construction of the proposed regularizer for such multiple $f_\A$'s is elaborated in Appendix~A.
We used $h_{\A,1}=0$ and $h_{\A,2}=\mathrm{Id}$ as the baseline functions.
The recognition networks, $g_{\A,1}$, $g_{\A,2}$, and $g_\P$, were modeled with MLPs.
We used the initial element of each $\bm{x}$ as an estimation of the initial condition $\vartheta_1$.

\paragraph{Results}
Figure~\ref{fig:pendulum_extp} demonstrates a unique benefit of the hybrid modeling.
We show an example of reconstruction with extrapolation.
Recall that the training data comprise sequences of range $0 \leq t < 2.5$ only; so the results in $t\geq2.5$ are extrapolation (in time) rather than mere reconstruction.
We can observe that while \texttt{NN+solver} cannot extrapolate even if it is equipped with an neural ODE, \texttt{NN+phys+reg} can reconstruct and extrapolate correctly.

Figure~\ref{fig:pendulum_cf} illustrates well the advantage of the proposed regularizers.
We show an example of generation from learned models with $z_\P$ manipulated.
Recall that $z_\P$ is expected to work as pendulum's angular velocity $\omega$.
We took a test sample with $\omega \approx \mathbb{E}[z_\P] \approx 2.15$ and generated signals with the original and different values of $z_\P$, keeping the values of $\bm{z}_\A$ to be the original posterior mean.
We can see that the generation from \texttt{NN+phys+reg} matches better with the signals from the true process.

Table~\ref{tab:pendulum_advdif_err} (left half) summarizes the performance in terms of the reconstruction error and the inference error of physics parameter $\omega$ on the test set.
The errors are reported in mean absolute errors (MAEs).
The inference error of $\omega$ is evaluated by $\vert \mathbb{E}[z_\P] - \omega_\text{true} \vert$.
\texttt{NN+phys+reg} achieves small values in \emph{both} reconstruction error and inference error.
Meanwhile, the MAE of reconstruction by \texttt{phys-only} is significantly worse than those of the other methods, and the MAE of $\omega$ inferred by \texttt{NN+phys} is significantly worse than the others.
These facts imply the effectiveness of the hybrid modeling and the proposed regularizers.



\subsection{Advection-diffusion system}
\label{expts:advdif}

\begin{table}[t]
    \renewcommand{\arraystretch}{0.9}
    \centering
    \caption{Reconstruction errors and inference errors on test sets of the pendulum data and the advection-diffusion data. Averages (and SDs) over 20 random trials are reported.}
    \label{tab:pendulum_advdif_err}
    \setlength{\tabcolsep}{3pt}
    {\small\begin{tabular}{m{1em}m{4.5em}M{3.4em}M{4.3em}M{3.4em}M{4.3em}M{3.6em}M{4.3em}M{3.6em}M{4.3em}}
        \toprule
        & & \multicolumn{4}{c}{Pendulum} & \multicolumn{4}{c}{Advection-diffusion}
        \\
        \cmidrule(lr){3-6}
        \cmidrule(lr){7-10}
        & & \multicolumn{2}{c}{MAE of reconst.} & \multicolumn{2}{c}{MAE of inferred $\omega$} & \multicolumn{2}{c}{MAE of reconst.} & \multicolumn{2}{c}{MAE of inferred $a$} \\
        \midrule
        \multicolumn{2}{l}{\texttt{NN-only}} &
            $0.438$ & \scriptsize{($2.9 \times 10^{-2}$)} &
            \multicolumn{2}{c}{--} &
            $0.0396$ & \scriptsize{($2.2 \times 10^{-4}$)} &
            \multicolumn{2}{c}{--}
            \\
        \multicolumn{2}{l}{\texttt{Phys-only}} &
            $1.55$ & \scriptsize{($7.1 \times 10^{-4}$)} &
            $0.232$ & \scriptsize{($5.9 \times 10^{-3}$)} &
            $0.393$ & \scriptsize{($9.5 \times 10^{-4}$)} &
            $0.0103$ & \scriptsize{($1.5 \times 10^{-3}$)}
            \\
        \multicolumn{2}{l}{\texttt{NN+solver}} &
            $0.439$ & \scriptsize{($2.3 \times 10^{-2}$)} &
            \multicolumn{2}{c}{--} &
            $0.0388$ & \scriptsize{($1.7 \times 10^{-4}$)} &
            \multicolumn{2}{c}{--}
            \\
        \multicolumn{2}{l}{\texttt{NN+phys}} &
            $0.370$ & \scriptsize{($4.3 \times 10^{-2}$)} &
            $1.04$ & \scriptsize{($2.2 \times 10^{-1}$)} &
            $0.0404$ & \scriptsize{($1.2 \times 10^{-2}$)} &
            $0.258$ & \scriptsize{($3.2 \times 10^{-1}$)}
            \\
        \multicolumn{2}{l}{\texttt{NN+phys+reg}} &
            $0.363$ & \scriptsize{($4.8 \times 10^{-2}$)} &
            $0.229$ & \scriptsize{($3.8 \times 10^{-2}$)} &
            $0.0437$ & \scriptsize{($1.5 \times 10^{-3}$)} &
            $0.00951$ & \scriptsize{($6.2 \times 10^{-3}$)}
            \\
            \midrule
        \parbox[t]{1em}{\multirow{3}{*}{\rotatebox[origin=c]{90}{\scriptsize Ablations}}}
        & $\alpha=0$ &
            $0.396$ & \scriptsize{($4.3 \times 10^{-2}$)} &
            $0.889$ & \scriptsize{($1.9 \times 10^{-1}$)} &
            $0.0461$ & \scriptsize{($1.3 \times 10^{-2}$)} &
            $0.0444$ & \scriptsize{($1.4 \times 10^{-2}$)}
            \\
        & $\beta=0$ &
            $0.372$ & \scriptsize{($4.1 \times 10^{-2}$)} &
            $0.223$ & \scriptsize{($3.6 \times 10^{-2}$)} &
            $0.0747$ & \scriptsize{($2.4 \times 10^{-2}$)} &
            $0.199$ & \scriptsize{($2.3 \times 10^{-1}$)}
            \\
        & $\gamma=0$ &
            $0.381$ & \scriptsize{($4.1 \times 10^{-2}$)} &
            $0.276$ & \scriptsize{($4.2 \times 10^{-2}$)} &
            $0.0588$ & \scriptsize{($9.1 \times 10^{-4}$)} &
            $0.0548$ & \scriptsize{($9.4 \times 10^{-7}$)}
            \\
        \bottomrule
    \end{tabular}}
    \renewcommand{\arraystretch}{1}
    \vspace*{-1.7ex}
\end{table}

\begin{figure}
    \begin{minipage}[t]{0.48\textwidth}
        \vspace*{0pt}
        \centering
\pgfplotsset{height=2cm,width=5.6cm}
\setlength{\tabcolsep}{2pt}
\begin{tabular}{m{4.8cm}m{1.5cm}}
    \begin{minipage}[c]{\linewidth}
        \begin{tikzpicture}
            \begin{axis}[compat=newest,
                axis on top,
                label style={font=\scriptsize},
                enlarge x limits=false,
                yticklabels={,,}, xticklabels={,,},
                ylabel=$s$,
                xmin=0, ymin=0, xmax=12.0, ymax=2.0,
                xtick pos=left, ytick pos=left,
                ]
                \addplot[] graphics [xmin=0,ymin=0,xmax=12.0,ymax=2.0] {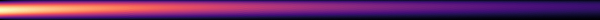};
                \addplot [white, thick, densely dashed, opacity=0.9] coordinates {(0.98, 0) (0.98, 2.0)};
            \end{axis}
        \end{tikzpicture}
    \end{minipage}
    & {\scriptsize Truth}
    \\[-9pt]
    \begin{minipage}[c]{\linewidth}
        \begin{tikzpicture}
            \begin{axis}[compat=newest,
                axis on top,
                label style={font=\scriptsize},
                enlarge x limits=false,
                yticklabels={,,}, xticklabels={,,},
                ylabel=$s$,
                xmin=0, ymin=0, xmax=12.0, ymax=2.0,
                xtick pos=left, ytick pos=left,
                ]
                \addplot[] graphics [xmin=0,ymin=0,xmax=12.0,ymax=2.0] {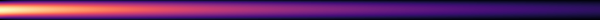};
                \addplot [white, thick, densely dashed, opacity=0.9] coordinates {(0.98, 0) (0.98, 2.0)};
            \end{axis}
        \end{tikzpicture}
    \end{minipage}
    & {\scriptsize\texttt{Phys-only}}
    \\[-9pt]
    \begin{minipage}[c]{\linewidth}
        \begin{tikzpicture}
            \begin{axis}[compat=newest,
                axis on top,
                label style={font=\scriptsize},
                enlarge x limits=false,
                xticklabels={,,}, yticklabels={,,}, 
                ylabel=$s$,
                xmin=0, ymin=0, xmax=12.0, ymax=2.0,
                xtick pos=left, ytick pos=left,
                ]
                \addplot[] graphics [xmin=0,ymin=0,xmax=12.0,ymax=2.0] {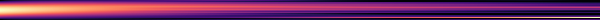};
                \addplot [white, thick, densely dashed, opacity=0.9] coordinates {(0.98, 0) (0.98, 2.0)};
            \end{axis}
        \end{tikzpicture}
    \end{minipage}
    & {\scriptsize\texttt{NN+solver}}
    \\[-9pt]
    \begin{minipage}[c]{\linewidth}
        \begin{tikzpicture}
            \begin{axis}[compat=newest,
                axis on top,
                label style={font=\scriptsize},
                enlarge x limits=false,
                yticklabels={,,}, xticklabels={,,},
                ylabel=$s$,
                xmin=0, ymin=0, xmax=12.0, ymax=2.0,
                xtick pos=left, ytick pos=left,
                ]
                \addplot[] graphics [xmin=0,ymin=0,xmax=12.0,ymax=2.0] {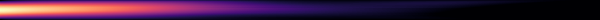};
                \addplot [white, thick, densely dashed, opacity=0.9] coordinates {(0.98, 0) (0.98, 2.0)};
            \end{axis}
        \end{tikzpicture}
    \end{minipage}
    & {\scriptsize\texttt{NN+phys}}
    \\[-9pt]
    \begin{minipage}[c]{\linewidth}
        \vspace*{0pt}
        \begin{tikzpicture}
            \begin{axis}[compat=newest,
                axis on top,
                label style={font=\scriptsize},
                enlarge x limits=false,
                yticklabels={,,}, 
                ylabel=$s$,
                xlabel=time $t$,
                xmin=0, ymin=0, xmax=12.0, ymax=2.0,
                xtick pos=left, ytick pos=left,
                xlabel style={at={(0.5,-1ex)}},
                ]
                \addplot[] graphics [xmin=0,ymin=0,xmax=12.0,ymax=2.0] {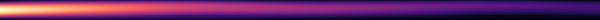};
                \addplot [white, thick, densely dashed, opacity=0.9] coordinates {(0.98, 0) (0.98, 2.0)};
            \end{axis}
        \end{tikzpicture}
    \end{minipage}
    & {\scriptsize\texttt{NN+phys+reg}\vspace*{2pt}}
\end{tabular}
\vspace*{-1ex}
\caption{Reconstruction and extrapolation of a test sample of the advection-diffusion data. Range $0 \leq t < 1$ is reconstruction, whereas $t \geq 1$ is extrapolation; dashed line is the border.}
\label{fig:advdif_extp}
    \end{minipage}
    \hfill
    \begin{minipage}[t]{0.48\textwidth}
        \vspace*{2pt}
        \input{fig_galaxy_random}
    \end{minipage}
    \vspace*{-2ex}
\end{figure}

\paragraph{Dataset}
We generated data from advection-diffusion PDE ${\partial T}/{\partial t} - a \cdot {\partial^2 T}/{\partial s^2} + b \cdot {\partial T}/{\partial s} = 0$, where $s$ is the 1-D spatial dimension.
We approximated the solution $T(s,t)$ on the 12-point even grid from $s=0$ to $s=s_\text{max}$, so each data-point $\bm{x}$ is a sequence of 12-dim vectors, i.e., $\bm{x} \coloneqq [\bm{T}_1 \ \cdots \ \bm{T}_\tau]\in\mathbb{R}^{12 \times \tau}$, where $\bm{T}_j \coloneqq [ T(0,t_j) \ \cdots \ T(s_\text{max},t_j) ]^\tr$ at $t_j \coloneqq (j-1)\Delta t$.
We set the boundary condition as $T(0,t)=T(s_\text{max},t)=0$ and the initial condition as $T(s,0)= c \sin(\pi s/s_\text{max})$.
We randomly drew $a$, $b$, and $c$ for each $\bm{x}$.
We generated 2,500 sequences with $\tau=50$ and $\Delta t = 0.02$ and separated them into a training, validation, and test sets with 1,000, 500, and 1,000 sequences, respectively.

\paragraph{Setting}
We set $f_\P$ as the diffusion PDE, i.e., $f_\P(T,z_\P) \coloneqq \partial T / \partial t - z_\P \partial^2 T / \partial s^2$, where $z_\P\in\mathbb{R}$ should work as diffusion coefficient $a$.
We augmented it by $f_\A(T, \bm{z}_\A)$ additively, where $f_\A$ was an MLP and $\bm{z}_\A \in \mathbb{R}^4$.
Hence, the decoding process is $\mathcal{F} \coloneqq \operatorname{solve}_T[f_\P(T,z_\P)+f_\A(T,\bm{z}_\A)=0]$.
We used $h_\A=0$ as the baseline function.
The recognition networks, $g_\A$ and $g_\P$, were modeled with MLPs.
We used the initial snapshot of each sequence $\bm{x}$ as an estimation of the initial condition $\bm{T}_1$.

\paragraph{Results}
Figure~\ref{fig:advdif_extp} shows an example of reconstruction with extrapolation.
As the training data only comprise sequences of range $0 \leq t < 1$, the remaining range $t \geq 1$ is extrapolation.
Only \texttt{NN+phys+reg} (the bottom panel) achieves adequate extrapolation; \texttt{phys-only} lacks advection, \texttt{NN+solver} has unnatural artifacts, and \texttt{NN+phys} infers $z_\P$ (i.e., diffusion coefficient $a$) wrongly.

Table~\ref{tab:pendulum_advdif_err} (right half) summarizes the reconstruction and inference errors, which are basically consistent with the results in the pendulum example, in the sense that \texttt{NN+phys+reg} achieves reasonable performance both in reconstruction and inference, while \texttt{phys-only} fails reconstruction, and \texttt{NN+phys} fails inference.
Note that the reconstruction performance of \texttt{NN+phys+reg} is slightly worse than some baselines, which is probably due to suboptimal hyperparameters.
In fact, with finer tuning of the hyperparameters, \texttt{NN+phys+reg} can achieve the reconstruction error closer to other methods while almost keeping the inference error\footnote{In the experiment with the advection-diffusion dataset reported in Table~\ref{tab:pendulum_advdif_err}, the selected values of the hyperparameters were $\alpha=0.1$, $\beta=0.01$, and $\gamma=10^6$, which were chosen from only eight candidates (see Appendix~E for detail). When we instead set $\alpha=0.032$, $\beta=0.01$, and $\gamma=10^6$ in the sensitivity experiment (shown in Appendix~F), the reconstruction error of \texttt{NN+phys+reg} was $0.0390$ ($4.5 \times 10^{-4}$), which is comparable to the baselines' performance in Table~\ref{tab:pendulum_advdif_err}. In this setting, the inference error of \texttt{NN+phys+reg} was $0.0103$ ($1.5 \times 10^{-3}$). We only reported the suboptimal values in Table~\ref{tab:pendulum_advdif_err} to align the granularity of the hyperparameter tuning grid with that in the experiment with the pendulum dataset.}.
We also show the performance of ablations of \texttt{NN+phys+reg}, where either of the regularizers was turned off (i.e., $\alpha=0$, $\beta=0$, or $\gamma=0$).
Not surprisingly their performance is worse than the full regularization, especially in terms of the inference error.


\subsection{Galaxy images}
\label{expts:galaxy}

\paragraph{Dataset}
We used images of galaxy of the Galaxy10 dataset \citep{leungDeepLearningMultielement2018}.
We selected the 589 images of the ``Disk, Edge-on, No Bulge'' class and separated them into training, validation, and test sets with 400, 100, and 89 images, respectively.
Each image is of size $69 \times 69$ with three channels.
We performed data augmentation with random rotation and increased the size of the training set by 20 times.

\paragraph{Setting}
We set $f_\P \colon \mathbb{R}_{>0}^{4} \to \mathbb{R}^{69 \times 69}$ as an exponential profile of the light distribution of galaxies \citep[see][and references therein]{aragon-calvoSelfsupervisedLearningPhysicsaware2020} whose input is $\bm{z}_\P \coloneqq [I_0 \ A \ B \ \vartheta]^\tr \in \mathbb{R}_{>0}^4$.
Let $[f_\P(\bm{z}_\P)]_{i,j}$ denote the $(i,j)$-element of the output of $f_\P$.
Then, for $1 \leq i,j \leq 69$, $[f_\P(\bm{z}_\P)]_{i,j} \coloneqq I_0 \exp(-r_{i,j})$, where $r_{i,j}^2 \coloneqq  (\mathsf{X}_j\cos\vartheta - \mathsf{Y}_i\sin\vartheta)^2/A^2 + (\mathsf{X}_j\sin\vartheta + \mathsf{Y}_i\cos\vartheta)^2/B^2$, and $(\mathsf{X}_j, \mathsf{Y}_i)$ is the coordinate on the $69 \times 69$ even grid on $[-1,1] \times [-1,1]$.
We modify the output of $f_\P$ using a U-Net-like neural network $f_\A \colon \mathbb{R}^{69 \times 69} \times \mathbb{R}^{\dim\bm{z}_\A} \to \mathbb{R}^{69 \times 69 \times 3}$.
Thus, the decoding process is $\mathcal{F} \coloneqq f_\A(f_\P(\bm{z}_\P), \bm{z}_\A)$.
We set $\dim\bm{z}_\A=2$ for \texttt{NN+phys+reg}.
We set $h_\A \colon \mathbb{R}^{69 \times 69} \to \mathbb{R}^{69 \times 69 \times 3}$ to be the $\operatorname{repeat}$ operator along the channel axis.
The encoding process is as follows: first, features are extracted from an image $\bm{x}$ by a convolutional net like \citep{aragon-calvoSelfsupervisedLearningPhysicsaware2020}.
The extracted features are flattened and fed to MLPs $g_\P$ and $g_\A$.

\paragraph{Results}
Figure~\ref{fig:galaxy_random} shows an example of original data and random generation from the learned models.
\texttt{NN-only} tends to generate non-realistic images, and \texttt{NN+phys} generates slightly better but still spuriously, whereas \texttt{NN+phys+reg} consistently generates galaxy-like images.
More results (reconstruction, counterfactual generation, and inspection of latent variable) are deferred to Appendix~F.


\subsection{Human gait}

\paragraph{Dataset}
We used a part of the dataset provided by \citep{lencioniHumanKinematicKinetic2019}, which contains measurements of locomotion at different speeds of 50 subjects.
We extracted the angles of hip, knee, and ankle in the sagittal plane.
Data originally comprise sequences of each stride normalized to be 100 steps, so each data-point $\bm{x}$ is a sequence $\bm{x} \coloneqq [ \bm\vartheta_1 \ \cdots \bm\vartheta_{100} ] \in \mathbb{R}^{3 \times 100}$, where $\bm\vartheta_j \coloneqq [\vartheta_{\text{hip},j} \ \vartheta_{\text{knee},j} \ \vartheta_{\text{ankle},j}]^\tr$.
We used different 400, 100, and 344 sequences as training, validation, and test sets, respectively.

\paragraph{Setting}
Biomechanical modeling of gait is a long-standing problem \citep[see, e.g.,][]{RMB2}.
We did not choose a specific model but let $f_\P$ be a trainable Hamilton's equation as in \citep{tothHamiltonianGenerativeNetworks2020,greydanusHamiltonianNeuralNetworks2019}.
$\bm{z}_\P\in\mathbb{R}^{2d_\mathrm{H}}$ worked as the initial conditions of it, where $d_\mathrm{H}$ was the dimensionality of the generalized position.
We let $d_\mathrm{H}=3$ and modeled the neural Hamiltonian with an MLP.
The solution of $f_\P=0$ was transformed by $f_\A$ that also took $\bm{z}_\A\in\mathbb{R}^{15}$ as an argument.
In summary, the decoding process is $\mathcal{F}=f_\A(\operatorname{solve}[f_\P=0], \bm{z}_\A)$.
We set $h_\A$ to be an affine transform at each timestep, which had a weight matrix and a bias as $\operatorname{param}(h)$.
The recognition networks were modeled with MLPs.

\paragraph{Results}
Figure~\ref{fig:hm_gait_rec} is for visually comparing the difference of the learned models' behavior due to the proposed regularizers.
We compare the reconstructions by \texttt{NN+phys} and \texttt{NN+phys+reg}.
The dashed lines show an intermediate of the decoding process, i.e., $\operatorname{solve}[f_\P=0]$, and the red solid lines show the final reconstruction, i.e., $f_\A(\operatorname{solve}[f_\P=0])$.
Without the regularization (upper row), $\operatorname{solve}[f_\P=0]$ returns almost meaningless signals, and $f_\A$ bears the most effort of reconstruction.
On the other hand, with the regularization (lower row), $\operatorname{solve}[f_\P=0]$ already matches well the data, and $f_\A$ modifies it only slightly.
Superiority of the regularized model was also confirmed quantitatively; the average test reconstruction errors were $0.273$ with \texttt{NN+phys} and $0.259$ with \texttt{NN+phys+reg}.

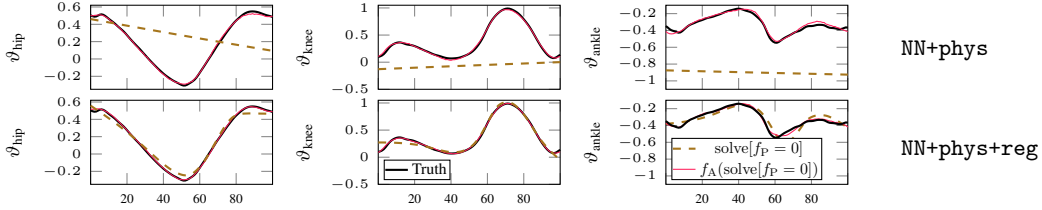
\begin{figure}
    \pgfplotsset{height=2.7cm,width=4cm}
    \begin{minipage}[c]{0.84\linewidth}
        \vspace*{0pt}
        \begin{tikzpicture}
            \begin{axis}[compat=newest,
                label style={font=\scriptsize}, ylabel={$\vartheta_\text{hip}$},
                enlarge x limits=false, ymin=-0.35, ymax=0.62,
                xticklabels={,,},
                ]
                \addplot [black, thick] table [x=norm_t, y=dat1] {hm_gait/reconstruction_idx0.txt};
                \addplot [myred] table [x=norm_t, y=noregPB1] {hm_gait/reconstruction_idx0.txt};
                \addplot [mybrown, thick, dashed] table [x=norm_t, y=noregP1] {hm_gait/reconstruction_idx0.txt};
            \end{axis}
        \end{tikzpicture}
        \begin{tikzpicture}
            \begin{axis}[compat=newest,
                label style={font=\scriptsize}, ylabel={$\vartheta_\text{knee}$},
                enlarge x limits=false, ymin=-0.5, ymax=1.05,
                xticklabels={,,},
                ]
                \addplot [black, thick] table [x=norm_t, y=dat2] {hm_gait/reconstruction_idx0.txt};
                \addplot [myred] table [x=norm_t, y=noregPB2] {hm_gait/reconstruction_idx0.txt};
                \addplot [mybrown, thick, dashed] table [x=norm_t, y=noregP2] {hm_gait/reconstruction_idx0.txt};
            \end{axis}
        \end{tikzpicture}
        \begin{tikzpicture}
            \begin{axis}[compat=newest,
                label style={font=\scriptsize}, ylabel={$\vartheta_\text{ankle}$},
                enlarge x limits=false, ymin=-1.1, ymax=-0.1,
                xticklabels={,,},
                ]
                \addplot [black, thick] table [x=norm_t, y=dat3] {hm_gait/reconstruction_idx0.txt};
                \addplot [myred] table [x=norm_t, y=noregPB3] {hm_gait/reconstruction_idx0.txt};
                \addplot [mybrown, thick, dashed] table [x=norm_t, y=noregP3] {hm_gait/reconstruction_idx0.txt};
            \end{axis}
        \end{tikzpicture}
    \end{minipage}
    \hfill
    \begin{minipage}[c]{0.14\linewidth}
        \vspace*{0pt}
        {\small \texttt{NN+phys}}
    \end{minipage}
    \\[-8pt]
    \begin{minipage}[c]{0.84\linewidth}
        \vspace*{0pt}
        \begin{tikzpicture}
            \begin{axis}[compat=newest,
                label style={font=\scriptsize}, ylabel={$\vartheta_\text{hip}$},
                enlarge x limits=false, ymin=-0.35, ymax=0.62,
                xticklabels={,,20,40,60,80,},
                ]
                \addplot [black, thick] table [x=norm_t, y=dat1] {hm_gait/reconstruction_idx0.txt};
                \addplot [myred] table [x=norm_t, y=regPB1] {hm_gait/reconstruction_idx0.txt};
                \addplot [mybrown, thick, dashed] table [x=norm_t, y=regP1] {hm_gait/reconstruction_idx0.txt};
            \end{axis}
        \end{tikzpicture}
        \begin{tikzpicture}
            \begin{axis}[compat=newest,
                label style={font=\scriptsize}, ylabel={$\vartheta_\text{knee}$},
                enlarge x limits=false, ymin=-0.5, ymax=1.05,
                xticklabels={,,20,40,60,80,},
                legend entries={Truth},
                legend style={at={(0.03,0.05)}, anchor=south west, nodes={scale=0.6, transform shape}, inner sep=0pt},
                legend image code/.code={
                    \draw[] 
                    plot coordinates {
                        (0cm,0cm)
                        (0.15cm,0cm)        
                        (0.3cm,0cm)         
                    };%
                },
                ]
                \addplot [black, thick] table [x=norm_t, y=dat2] {hm_gait/reconstruction_idx0.txt};
                \addplot [myred] table [x=norm_t, y=regPB2] {hm_gait/reconstruction_idx0.txt};
                \addplot [mybrown, thick, dashed] table [x=norm_t, y=regP2] {hm_gait/reconstruction_idx0.txt};
            \end{axis}
        \end{tikzpicture}
        \begin{tikzpicture}
            \begin{axis}[compat=newest,
                label style={font=\scriptsize}, ylabel={$\vartheta_\text{ankle}$},
                enlarge x limits=false, ymin=-1.1, ymax=-0.1,
                xticklabels={,,20,40,60,80,},
                legend entries={$\operatorname{solve}[f_\P=0]$, $f_\A(\operatorname{solve}[f_\P=0])$},
                legend style={at={(0.03,0.05)}, anchor=south west, nodes={scale=0.6, transform shape}, inner sep=0pt},
                legend columns=1,
                legend image code/.code={
                    \draw[] 
                    plot coordinates {
                        (0cm,0cm)
                        (0.15cm,0cm)        
                        (0.3cm,0cm)         
                    };%
                },
                ]
                \addplot [mybrown, thick, dashed] table [x=norm_t, y=regP3] {hm_gait/reconstruction_idx0.txt};
                \addplot [myred] table [x=norm_t, y=regPB3] {hm_gait/reconstruction_idx0.txt};
                \addplot [black, thick] table [x=norm_t, y=dat3] {hm_gait/reconstruction_idx0.txt};
            \end{axis}
        \end{tikzpicture}
    \end{minipage}
    \hfill
    \begin{minipage}[c]{0.14\linewidth}
        \vspace*{0pt}
        {\small \texttt{NN+phys+reg}}
    \end{minipage}
    \caption{Reconstruction of a test sample of the gait data. Horizontal axis is normalized time.}
    \label{fig:hm_gait_rec}
    \vspace*{-2ex}
\end{figure}

\section{Conclusion}

Physics-integrated VAEs by construction attain partial interpretability as some of the latent variables are semantically grounded to the physics models, and thus we can generate signals in a controlled manner.
Moreover, they have extrapolation capability due to the physics models.
In this work, we proposed a regularized learning objective for ensuring a proper functionality of the integrated physics models.
We empirically validated the aforementioned unique capability of physics-integrated VAEs and the importance of the proposed regularization method.
In future studies, it would be interesting to investigate possibility and extension to learn a hybrid generative model with a highly complex observation process.

\begin{ack}
This work was supported by the Innosuisse project \emph{Industrial artificial intelligence for intelligent machines and manufacturing digitalization} (39453.1 IP-ICT) and the Swiss National Science Foundation Sinergia project \emph{Modeling pathological gait resulting from motor impairments} (CRSII5\_177179).
\end{ack}

\bibliographystyle{abbrvnat}
\bibliography{main,books,additional,add2}

\appendix

\section{General description of physics-integrated VAEs}
\label{general}

In this section, we provide a general description of the physics-integrated VAEs and the proposed regularization method, since we only described a simple case in Sections~2 and 3 of the main text.
The main difference of the general description from the simple one is the number of trainable function $f_\A$ in the model.

\subsection{Model}

We here consider a generalized case in which we have multiple trainable models $f_{\A,1}, f_{\A,2}, \dots, f_{\A,K}$.
We fix the number of $f_\P$ to be one as in the main text for clarity, while an extension in this regard is straightforward.
We exemplify some use cases with multiple $f_\A$'s in Appendix~\ref{morerel}.


\subsubsection{Latent variables}

Beside $\bm{z}_\P\in\mathcal{Z}_\P$, we consider $\bm{z}_{\A,k} \in \mathcal{Z}_{\A,k}$ for $k=1,\dots,K$.
If $f_{\A,k}$ does not take $\bm{z}$ as argument for some $k$, we simply suppose $\mathcal{Z}_{\A,k}=\emptyset$ for such $k$.
Otherwise, we suppose that $\mathcal{Z}_{\A,k}$ is (some subset of) the Euclidean space for simplicity of discussion.
The prior distributions are:
\begin{equation}\label{eq:prior_P}
    p(\bm{z}_\P) \coloneqq \mathcal{N}(\bm{z}_\P \mid \bm{m}_\P, v^2_\P \bm{I}),
\end{equation}
and
\begin{equation}\label{eq:prior_A}
    p(\bm{z}_{\A,k}) \coloneqq \mathcal{N}(\bm{z}_{\A,k} \mid \bm{0}, \bm{I}),
\end{equation}
for $k$ whose $\mathcal{Z}_{\A,k}$ is not empty.

\subsubsection{Decoder}

We intentionally do not specify the ranges and the domains of $f_\P$ and $f_{\A,1}, f_{\A,2}, \dots, f_{\A,K}$ because they depend on how these functions are connected each other.
We denote the decoding process again with a functional $\mathcal{F}$ whose arguments are $f_\P$ and $f_{\A,1}, \dots, f_{\A,K}$ as well as $\bm{z}$'s, that is, $\mathcal{F}[ f_\P, f_{\A,1}, \dots, f_{\A,K}; \bm{z}_\P, \bm{z}_{\A,1},\dots,\bm{z}_{\A,K} ]$\footnote{Note that the expression in Section~2 of the main text, $\mathcal{F}[f_\A(f_\P(\bm{z}_\P), \bm{z}_\A)]$, violates this general notation; for consistency, it should have been $\mathcal{F}[f_\P, f_\A; \bm{z}_\P, \bm{z}_\A]$ instead. The idea there was to emphasize the fact that $f_\A$ and $f_\P$ are \emph{somehow} (not only additively) composited in the model.}.
Inside $\mathcal{F}$ the functions can be connected in various ways; $\mathcal{F}$ can include 1) \emph{in-equation} augmentation $\operatorname{solve}(f_\P + f_\A = 0)$ or $\operatorname{solve}(f_\A \circ f_\P = 0)$, 2) \emph{out-equation} augmentation $f_\A(\operatorname{solve}(f_\P = 0))$, and 3) their arbitrary combinations, e.g., $f_{\A,3}(\operatorname{solve}(f_{\A,2}(f_\P + f_{\A,1}) = 0))$.
We show some examples in Appendix~\ref{morerel}.
The observation model is
\begin{equation}\label{eq:dec_general}
    p_\theta(\bm{x} \mid \bm{z}_\P, \bm{z}_{\A,1},\dots,\bm{z}_{\A,K}) \coloneqq \mathcal{N} \big( \bm{x} \mid \mathcal{F}[ f_\P, f_{\A,1}, \dots, f_{\A,K}; \bm{z}_\P, \bm{z}_{\A,1},\dots,\bm{z}_{\A,K} ], \bm\Sigma_x \big),
\end{equation}
where $\theta$ is the set of trainable parameters of $f_\P$ and $f_{\A,1},\dots,f_{\A,K}$ (and $\bm\Sigma_x$).

\subsubsection{Encoder}

Accordingly, the approximated posterior is
\begin{equation}\label{eq:enc_general}
    q_\psi(\bm{z}_\P, \bm{z}_{\A,1},\dots,\bm{z}_{\A,K} \mid \bm{x})
    \coloneqq q_\psi (\bm{z}_{\A,1},\dots,\bm{z}_{\A,K} \mid \bm{x}) q_\psi (\bm{z}_\P \mid \bm{x}, \bm{z}_{\A,1},\dots,\bm{z}_{\A,K}).
\end{equation}
We do not specify further structures of $q_\psi (\bm{z}_{\A,1},\dots,\bm{z}_{\A,K} \mid \bm{x})$ and  $q_\psi (\bm{z}_\P \mid \bm{x}, \bm{z}_{\A,1},\dots,\bm{z}_{\A,K})$ because they depend on use cases.
We denote the recognition networks for $\bm{z}_\P$ and $\bm{z}_{\A,k}$ by $g_\P$ and $g_{\A,k}$, respectively for $k=1,\dots,K$.
$\psi$ is again the set of all the trainable parameters in the encoder side of the model.

\subsection{Regularizers}

We slightly modify the definition of the proposed regularizers in accordance with the general description of the model.

The regularizer to suppress trainable components, $R_\text{PPC}$, should be able to measure the contribution of all the trainable components, $f_{\A,1},\dots,f_{\A,K}$.
While the original definition in Section~3 of the main text would still work as is, we empirically found that the following modification was useful in some cases.
The idea is to consider the \emph{marginal contribution} (compared to the physics model) of \emph{each} of the trainable components, $f_{\A,1},\dots,f_{\A,K}$, instead of computing the contribution of all $f_\A$'s altogether.
To show the essence of the idea, let us suppose $K=2$.
We consider the discrepancy between posterior predictive distributions for the following combinations:
\begin{gather}
    D_\mathrm{KL} \big[ p_{\theta,\psi}(\tilde{\bm{x}} \mid X) \Mid p_{\theta^\mathrm{r},\psi}^{\mathrm{r},\{1\}}(\tilde{\bm{x}} \mid X) \big],
    \label{eq:pcc_1}
    \\
    D_\mathrm{KL} \big[ p_{\theta,\psi}(\tilde{\bm{x}} \mid X) \Mid p_{\theta^\mathrm{r},\psi}^{\mathrm{r},\{2\}}(\tilde{\bm{x}} \mid X) \big],
    \label{eq:pcc_2}
    \\
    D_\mathrm{KL} \big[ p_{\theta^\mathrm{r},\psi}^{\mathrm{r},\{1\}}(\tilde{\bm{x}} \mid X) \Mid p_{\theta^\mathrm{r},\psi}^{\mathrm{r},\{1,2\}}(\tilde{\bm{x}} \mid X) \big],
    \label{eq:pcc_3}
    \\
    D_\mathrm{KL} \big[ p_{\theta^\mathrm{r},\psi}^{\mathrm{r},\{2\}}(\tilde{\bm{x}} \mid X) \Mid p_{\theta^\mathrm{r},\psi}^{\mathrm{r},\{1,2\}}(\tilde{\bm{x}} \mid X) \big],
    \label{eq:pcc_4}
\end{gather}
where $p_{\theta^\mathrm{r},\psi}^{\mathrm{r},\mathcal{I}}(\tilde{\bm{x}} \mid X)$ ($I \subseteq \{1,\dots,K\}$) is a \emph{partial} physics-only reduced model in which $f_{A,i}, \forall i \in \mathcal{I}$ are replaced with baseline function $h_{\A,i}$.
We let $p_{\theta^\mathrm{r},\psi}^{\mathrm{r},\mathcal{I}=\emptyset}(\tilde{\bm{x}} \mid X) \coloneqq p_{\theta,\psi}(\tilde{\bm{x}} \mid X)$ for convenience of notation.

Let us denote the upper bounds (see Proposition~1) of Eqs.~\eqref{eq:pcc_1}--\eqref{eq:pcc_4} respectively as follows:
\begin{gather*}
    \mathbb{E}_{p_\mathrm{d}(\bm{x} \mid X)}\hat{D}_{\emptyset,\{1\}}(\theta,\operatorname{param}(h),\psi;\bm{x}),
    \\
    \mathbb{E}_{p_\mathrm{d}(\bm{x} \mid X)}\hat{D}_{\emptyset,\{2\}}(\theta,\operatorname{param}(h),\psi;\bm{x}),
    \\
    \mathbb{E}_{p_\mathrm{d}(\bm{x} \mid X)}\hat{D}_{\{1\},\{1,2\}}(\theta,\operatorname{param}(h),\psi;\bm{x}),
    \\
    \mathbb{E}_{p_\mathrm{d}(\bm{x} \mid X)}\hat{D}_{\{2\},\{1,2\}}(\theta,\operatorname{param}(h),\psi;\bm{x}).
\end{gather*}
Then, the regularizer is defined as
\begin{equation}\begin{aligned}
    &4R_\text{PPC}(\theta,\operatorname{param}(h),\psi)
    \\
    &\quad\coloneqq
        \mathbb{E}_{p_\mathrm{d}(\bm{x} \mid X)}\hat{D}_{\emptyset,\{1\}}(\theta,\operatorname{param}(h),\psi;\bm{x})
        + \mathbb{E}_{p_\mathrm{d}(\bm{x} \mid X)}\hat{D}_{\emptyset,\{2\}}(\theta,\operatorname{param}(h),\psi;\bm{x})
    \\
    &\qquad
        + \mathbb{E}_{p_\mathrm{d}(\bm{x} \mid X)}\hat{D}_{\{1\},\{1,2\}}(\theta,\operatorname{param}(h),\psi;\bm{x})
        + \mathbb{E}_{p_\mathrm{d}(\bm{x} \mid X)}\hat{D}_{\{2\},\{1,2\}}(\theta,\operatorname{param}(h),\psi;\bm{x}).
\end{aligned}\end{equation}

The regularizer to use physics-based data augmentation, $R_\text{DA}$, is defined in almost the same way as in the simple case --- we draw samples $\bm{z}^\star_\P$ from some distribution of $\bm{z}_\P$ and generate physics-only augmentation by $\bm{x}^\mathrm{r}(\bm{z}^\star_\P) \coloneqq \mathcal{F}[ f_\P, h_{\A,1}, \dots, h_{\A,K}; \bm{z}_\P^\star ]$.
Note that all of $f_\A$'s are replaced with $h_\A$'s at once unlike the aforementioned case of $R_\text{PPC}$.
\section{Proof of Proposition~1}
\label{proof}


We use the following well-known facts in deriving the upper bound in Proposition~1.

\begin{lemma}
    Let $p_1(x,y)$ and $p_2(x,y)$ be two joint distributions on random variables $x$ and $y$, and $p_1(x)$ and $p_2(x)$ be the corresponding marginals.
    Then,
    \begin{equation}\label{eq:klub}
        D_\mathrm{KL}[ p_1(x) \Mid p_2(x) ] \leq D_\mathrm{KL}[ p_1(x,y) \Mid p_2(x,y) ].
    \end{equation}
\end{lemma}
\begin{proof}
    From definition,
    \begin{equation*}\begin{aligned}
        D_\mathrm{KL} [ p_1(x,y) \Mid p_2(x,y) ]
        &= \int p_1(x,y) \frac{p_1(x,y)}{p_2(x,y)} \mathrm{d}x \mathrm{d}y
        \\
        &= \int p_1(y \mid x) p_1(x) \frac{p_1(y \mid x) p_1(x)}{p_2(y \mid x) p_2(x)} \mathrm{d}x \mathrm{d}y
        \\
        &= \int p_1(y \mid x) p_1(x) \frac{p_1(y \mid x)}{p_2(y \mid x)} \mathrm{d}x \mathrm{d}y
        + \int p_1(y \mid x) p_1(x) \frac{p_1(x)}{p_2(x)} \mathrm{d}x \mathrm{d}y
        \\
        &= \int p_1(x) \left( \int p_1(y \mid x) \frac{p_1(y \mid x)}{p_2(y \mid x)} \mathrm{d}y \right) \mathrm{d}x
        + \int p_1(x) \frac{p_1(x)}{p_2(x)} \mathrm{d}x
        \\
        &= \mathbb{E}_{p_1(x)} D_\mathrm{KL} [ p_1(y \mid x) \Mid p_2(y \mid x) ] + D_\mathrm{KL} [ p_1(x) \Mid p_2(x) ].
    \end{aligned}\end{equation*}
    Hence, from the nonnegativity of the KL divergence, we have
    \begin{equation*}\begin{aligned}
        D_\mathrm{KL} [ p_1(x) \Mid p_2(x) ]
        &= D_\mathrm{KL} [ p_1(x,y) \Mid p_2(x,y) ] - \mathbb{E}_{p_1(x)} D_\mathrm{KL} [ p_1(y \mid x) \Mid p_2(y \mid x) ]
        \\
        &\leq D_\mathrm{KL} [ p_1(x,y) \Mid p_2(x,y) ].
    \end{aligned}\end{equation*}
\end{proof}

\begin{lemma}
    Let $x$ and $y$ be random variables with joint distribution $q(x,y)$.
    Let $I(x; y)$ be the mutual information between $x$ and $y$, i.e.: $I(x; y) \coloneqq D_\mathrm{KL} [ q(x,y) \Mid q(x) q(y) ]$.
    Let $p(x)$ be some distribution of $x$.
    Then,
    \begin{equation}\label{eq:miub}
        I(x; y) \leq \mathbb{E}_{q(y)} D_\mathrm{KL} \big[ q(x \mid y) \Mid p(x) \big].
    \end{equation}
\end{lemma}
\begin{proof}
    From the nonnegativity of the KL divergence,
    \begin{equation*}\begin{aligned}
        I(x,y)
        &= D_\mathrm{KL} [ q(x,y) \Mid q(x) q(y) ]
        \\
        &= \int q(x,y) \log \frac{q(x,y)}{q(x) q(y)} \mathrm{d}x \mathrm{d}y
        \\
        &= \int q(x,y) \log\frac{q(x \mid y)}{q(x)} \mathrm{d}x \mathrm{d}y
        \\
        &= \int q(x,y) \log\frac{q(x \mid y) p(x)}{p(x) q(x)} \mathrm{d}x \mathrm{d}y
        \\
        &= \mathbb{E}_{q(y)} D_\mathrm{KL} \big[ q(x \mid y) \Mid p(x) \big] - D_\mathrm{KL} \big[ q(x) \Mid p(x) \big]
        \\
        &\leq \mathbb{E}_{q(y)} D_\mathrm{KL} \big[ q(x \mid y) \Mid p(x) \big].
    \end{aligned}\end{equation*}
\end{proof}

Now we give a proof of Proposition~1.
\begin{proof}[Proof of Proposition~1]
    Let us denote the set of $\bm{z}_\P$ and $\bm{z}_\A$ by $z$.
    As a posterior predictive distribution $p(\tilde{\bm{x}} \mid X)$ is obtained by marginalizing out $z$ and $\bm{x}$ of joint distribution $p(\tilde{\bm{x}}, z, \bm{x} \mid X)$, from \eqref{eq:klub},
    \begin{equation}\label{eq:klub_apply}
         D_\mathrm{KL} \big[ p_{\theta,\psi}(\tilde{\bm{x}} \mid X) \Mid p_{\theta^\mathrm{r},\psi}^\mathrm{r}(\tilde{\bm{x}} \mid X) \big]
         \leq
         D_\mathrm{KL} \big[ p_{\theta,\psi}(\tilde{\bm{x}}, z, \bm{x} \mid X) \Mid p_{\theta^\mathrm{r},\psi}^\mathrm{r}(\tilde{\bm{x}}, z, \bm{x} \mid X) \big].
    \end{equation}
    The right-hand side of \eqref{eq:klub_apply} is
    \begin{equation*}\begin{aligned}
        &
        D_\mathrm{KL} \big[ p_{\theta,\psi}(\tilde{\bm{x}}, z, \bm{x} \mid X) \Mid p_{\theta^\mathrm{r},\psi}^\mathrm{r}(\tilde{\bm{x}}, z, \bm{x} \mid X) \big]
        \\&\quad
        = D_\mathrm{KL} \Big[
            p_\theta (\tilde{\bm{x}} \mid z)
            q_\psi (z \mid \bm{x})
            p_\mathrm{d}(\bm{x} \mid X)
        \,\Big\Vert\,
            p_{\theta^\mathrm{r}}^\mathrm{r} (\tilde{\bm{x}} \mid z)
            q_\psi^\mathrm{r} (z \mid \bm{x})
            p_\mathrm{d}(\bm{x} \mid X)
        \Big]
        \\&\quad
        =
            \mathbb{E}_{p_\mathrm{d}(\bm{x} \mid X)} \mathbb{E}_{q_\psi (z \mid \bm{x})} D_\mathrm{KL} \big[ p_\theta (\tilde{\bm{x}} \mid z) \Mid p_{\theta^\mathrm{r}}^\mathrm{r} (\tilde{\bm{x}} \mid z) \big]
            + \mathbb{E}_{p_\mathrm{d}(\bm{x} \mid X)} D_\mathrm{KL} \big[ q_\psi (z \mid \bm{x}) \Mid q_\psi^\mathrm{r} (z \mid \bm{x}) \big],
    \end{aligned}\end{equation*}
    where the last term is
    \begin{equation*}\begin{aligned}
        &
        \mathbb{E}_{p_\mathrm{d}(\bm{x} \mid X)} D_\mathrm{KL} \big[ q_\psi (z \mid \bm{x}) \Mid q_\psi^\mathrm{r} (z \mid \bm{x}) \big]
        \\
        &\quad
        = \mathbb{E}_{p_\mathrm{d}(\bm{x} \mid X)} D_\mathrm{KL} \big[ q_\psi (\bm{z}_\P \mid \bm{x}, \bm{z}_\A) q_\psi(\bm{z}_\A \mid \bm{x}) \Mid q_\psi (\bm{z}_\P \mid \bm{x}) p(\bm{z}_\A) \big]
        \\
        &\quad
        = \mathbb{E}_{p_\mathrm{d}(\bm{x} \mid X)} \Big[ \mathbb{E}_{q_\psi (\bm{z}_\A \mid \bm{x})} D_\mathrm{KL} \big[ q_\psi (\bm{z}_\P \mid \bm{x}, \bm{z}_\A) \Mid q_\psi (\bm{z}_\P \mid \bm{x}) \big] + D_\mathrm{KL} \big[ q_\psi(\bm{z}_\A \mid \bm{x}) \Mid p(\bm{z}_\A) \big] \Big]
        \\
        &\quad
        = \mathbb{E}_{p_\mathrm{d}(\bm{x} \mid X)} \Big[ I(\bm{z}_\P; \bm{z}_\A) + D_\mathrm{KL} \big[ q_\psi(\bm{z}_\A \mid \bm{x}) \Mid p(\bm{z}_\A) \big] \Big].
    \end{aligned}\end{equation*}
    Hence, from the upper bound of mutual information, \eqref{eq:miub}, the right-hand side of \eqref{eq:klub_apply} is further upper bounded as 
    \begin{equation*}\begin{aligned}
        &
        D_\mathrm{KL} \big[ p_{\theta,\psi}(\tilde{\bm{x}}, z, \bm{x} \mid X) \Mid p_{\theta^\mathrm{r},\psi}^\mathrm{r}(\tilde{\bm{x}}, z, \bm{x} \mid X) \big]
        \\&\quad
        \leq \mathbb{E}_{p_\mathrm{d}(\bm{x} \mid X)} \Big[
            \mathbb{E}_{q_\psi (z \mid \bm{x})} D_\mathrm{KL} \big[ p_\theta (\tilde{\bm{x}} \mid z) \Mid p_{\theta^\mathrm{r}}^\mathrm{r} (\tilde{\bm{x}} \mid \bm{z}_\P, \bm{z}_\A) \big]
            \\&\qquad\quad
            + \mathbb{E}_{q_\psi(\bm{z}_\A \mid \bm{x})} D_\mathrm{KL} \big[ q_\psi(\bm{z}_\P \mid \bm{x}, \bm{z}_\A) \Mid p(\bm{z}_\P) \big]
            + D_\mathrm{KL} \big[ q_\psi(\bm{z}_\A \mid \bm{x}) \Mid p(\bm{z}_\A) \big]
        \Big].
    \end{aligned}\end{equation*}
\end{proof}
\section{Additional remarks on the regularized learning method}
\label{remark}

\paragraph{Upper bound of KL in general case}
In the general case of Appendix~\ref{general}, the upper bound of the KL divergence used for defining $R_\text{PPC}$ becomes slightly different.
For example, a bound of \eqref{eq:pcc_1} is as follows (recall that we focused the case of $K=2$ for discussion):
\begin{multline*}
    D_\mathrm{KL} \big[ p_{\theta,\psi}(\tilde{\bm{x}} \mid X) \Mid p_{\theta^\mathrm{r},\psi}^{\mathrm{r},\{1\}}(\tilde{\bm{x}} \mid X) \big]
    \leq
    \mathbb{E}_{p_\mathrm{d}(\bm{x} \mid X)} \Big[
        \mathbb{E}_{q_\psi(\bm{z}_\P, \bm{z}_\A \mid \bm{x})} D_\mathrm{KL} [ p_\theta \Mid p_\theta^{\mathrm{r},\{1\}} ]
        \\
        + D_\mathrm{KL} [ q_\psi(\bm{z}_{\A,1}, \bm{z}_{\A,2} \mid \bm{x}) \Mid p_{\A,\{1,2\}} ]
        + \mathbb{E}_{q_\psi(\bm{z}_{\A,1}, \bm{z}_{\A,2} \mid \bm{x})} D_\mathrm{KL} [ q_\psi(\bm{z}_\P \mid \bm{z}_{\A,1}, \bm{z}_{\A,2}, \bm{x}) \Mid p_\P ]
    \Big],
\end{multline*}
where $p_{\A,\{1,2\}}$ is some distribution of $\bm{z}_{\A,1}$ and $\bm{z}_{\A,2}$, for example $p_{\A,\{1,2\}}=p(\bm{z}_{\A,1})p(\bm{z}_{\A,2})$ using priors.
This upper bound can be derived analogously to Proposition~1.

\paragraph{Interpretation of upper bound}
It is interesting that the mutual information $I(\bm{z}_\P; \bm{z}_\A)$ appears in the intermediate bound of $D_\mathrm{KL} \big[ p_{\theta,\psi}(\tilde{\bm{x}} \mid X) \Mid p_{\theta^\mathrm{r},\psi}^\mathrm{r}(\tilde{\bm{x}} \mid X) \big]$ (see the proof of Proposition~1).
Such a mutual information becomes a conditional mutual information (e.g., $I(\bm{z}_\P; \bm{z}_{\A,1} \mid \bm{z}_{\A,2})$) in the general case.
Moreover, the last two terms of the upper bound in Proposition~1 are the same as the last two terms of the ELBO when $p_\P$ and $p_\A$ are the priors.
In such a case, adding them as regularizers to the objective is equivalent to what is done in $\beta$-VAE \citep{higginsVVAELearningBasic2017}.
It would also be interesting to discuss connection with the work by \citet{zhaoInformationAutoencodingFamily2018}.

\paragraph{Usage of augmented data}
Data augmented with physics-based prior knowledge can also be used for pretraining (e.g., \citet{jiaPhysicsGuidedRNNs2019}).
We rather generate and use them during the main training procedure as regularizers because the effects of pretraining may diminish in the main training.

\section{Related work}
\label{morerel}

We introduce related studies that could not be in Section~4 of the main text due to length limit.
Recall that in Section~4, we reviewed the studies with the following two perspectives: ``Physics+ML in model design'' and ``Physics+ML in objective design.''
In this appendix, we follow a slightly different taxonomy: 1 ) \emph{physics-integrated}, 2) \emph{physics-informed}, and 3) \emph{physics-inspired} methods.
The first two of these three roughly correspond to the two perspectives in Section~4 of the main text.
In contrast, we did not focus on the last one, physics-inspired method, in Section~4, while it will be informative for readers to provide a broader view of the context.
We refer to some reviews and surveys on these topics, such as ones by \citet{willardIntegratingPhysicsbasedModeling2020,vonruedenInformedMachineLearning2020,von_rueden_combining_2020,beckh_explainable_2021}, and \citet{karniadakis_physics-informed_2021}.
We would like to emphasize that the aforementioned three areas of research are never exclusive, and study that can bridge and unify them will be important.

\subsection{Physics-integrated methods}

We refer to methods where the model is a combination of physics models and machine learning models as \emph{physics-integrated}\footnote{Though this has been traditionally known as gray-box modeling, here we put an emphasis on the focus on physics-based models and adjust the wording with other related perspectives.} ones.
As such an approach was already explained to some extent in Section~4 of the main text, we here focus on exemplifying architectures of physics-integrated models.
Most of the studies referred to here did not aim generative modeling originally, though the ideas can be fitted to our general architecture of physics-integrated VAEs.
For more information, we recommend consulting the excellent survey / overview papers \citep[e.g.,][]{thompsonModelingChemicalProcesses1994,karpatneTheoryguidedDataScience2017,renLearningWeakSupervision2018,reichstein_deep_2019,wangIntegratingModeldrivenDatadriven2019,camps-vallsLivingPhysicsMachine2020,shlezingerModelbasedDeepLearning2020,willardIntegratingPhysicsbasedModeling2020,sch21}.

\paragraph{In-equation augmentation}
A numerical solver of dynamics models such as ODEs, PDEs, and discrete-time difference equations are one of the most prevailing forms of an equation-solving process that can be in a physics-integrated VAE.
In such cases, $f_\P$ and/or $f_\A$ would give terms that appear in a dynamics equation.
They are combined additively in many cases \citep{rico-martinezContinuoustimeNonlinearSignal1994,thompsonModelingChemicalProcesses1994,reinhartHybridAnalyticalDatadriven2017,golemoSimtorealTransferNeuralaugmented2018,mehtaNeuralDynamicalSystems2020,dechelleBridgingDynamicalModels2020,roehrlModelingSystemDynamics2020,leguenDisentanglingPhysicalDynamics2020,leguenAugmentingPhysicalModels2021,vianaEstimatingModelInadequacy2021,mitusch_hybrid_2021,qian_integrating_2021}, for example:
\begin{equation}
    \mathcal{F} \coloneqq \operatorname{solve}_y \big[ f_\P(y,\bm{z}_\P) + f_\A(y,\bm{z}_\A) = 0 \big],
\end{equation}
where $\operatorname{solve}_y$ refers to a numerical ODE/PDE solver with regard to $y$ and returns the value of the solution on some time/space grid.
Another way of combining $f_\P$ and $f_\A$ in this context is composition \citep{psichogiosHybridNeuralNetworkfirst1992,thompsonModelingChemicalProcesses1994,longHybridNetIntegratingModelbased2018,wanDataassistedReducedorderModeling2018,lanusseHybridPhysicaldeepLearning2019,debezenacDeepLearningPhysical2019,mohanEmbeddingHardPhysical2020,mateiInterpretableMachineLearning2020,behjatPhysicsawareLearningArchitecture2020,arikInterpretableSequenceLearning2020,liKohnShamEquationsRegularizer2020,jiangSimGANHybridSimulator2021,karra_adjointnet_2021,djeumou_neural_2021}, for example:
\begin{equation}
    \mathcal{F} \coloneqq \operatorname{solve}_y \big[ f_\P(y, \bm{z}_\P, f_\A(y,\bm{z}_\A)) = 0 \big],
\end{equation}
where $f_\A$ often gives estimation of some unknown or varying physics parameters in $f_\P$.
The order of the composition may reverse \citep[recent examples include][]{ajayAugmentingPhysicalSimulators2018,ajayCombiningPhysicalSimulators2019}, that is,
\begin{equation}
    \mathcal{F} \coloneqq \operatorname{solve}_y \big[ f_\A(y, \bm{z}_\A, f_\P(y,\bm{z}_\P)) = 0 \big],
\end{equation}
where the output of a physics model is augmented by a machine learning model.
Such a mechanism is often called \emph{residual physics}.
Some studies consider more complex combinations of $f_\P$ and $f_\A$, for example, $\mathcal{F} \coloneqq \operatorname{solve} [f_{\P,2}( f_\A(f_{\P,1}) ) = 0]$ \citep{raissiDeepHiddenPhysics2018,degrooteNeuralNetworkAugmented2019,heidenNeuralSimAugmentingDifferentiable2020,kaltenbachPhysicsawareProbabilisticModel2021}.
A trickier case appears in \citet{jiangDataaugmentedContactModel2018}, where discrete state of contact dynamics is first determined by a data-driven classifier, which is then used for choosing one of physics models (also including trainable ones) to be used.
Moreover, \citet{um_solver---loop_2020} considered to correct numerical errors by neural nets inside a differentiable solver of differential equations.

The equation-solving process can be anything else than an ODE/PDE solver.
If (augmented) physics models are algebraic equations with closed-form solutions, $\mathcal{F}$ just evaluates some functions \citep[e.g.,][]{aragon-calvoSelfsupervisedLearningPhysicsaware2020}.
If no closed-form solution is available, a diffentiable optimizer may be utilized in $\mathcal{F}$.

We also note that the latent force models \citep{alvarezLatentForceModels2009} are known as a principled method to incorporate physics models in differential equations into Gaussian processes.

\paragraph{Out-equation augmentation}
Physics and machine learning integration can also happen outside an equation-solving process.
The simplest case is
\begin{equation}
    \mathcal{F} \coloneqq f_\A( \operatorname{solve}[\cdots], \bm{z}_\A )
    \quad\text{or}\quad
    \mathcal{F} \coloneqq f_\A(\operatorname{solve}[\cdots], \bm{z}_\A) + \operatorname{solve}[\cdots],
\end{equation}
where $\operatorname{solve}[\cdots]$ denotes the output of some equation-solving process, which also includes $f_\P$ as well as another set of $f_\A$'s.
For example, such architectures can be found in the following use cases:
\begin{itemize}[itemsep=0pt,topsep=0pt,leftmargin=*]
    \item $f_\A$ corrects the output of an equation-solving process, $\operatorname{solve}[\cdots]$, to compensate inaccuracy of physics models or unmodeled phenomena \citep{youngPhysicallyBasedMachine2017,chenPGAPhysicsGuided2018,nutkiewiczDatadrivenUrbanEnergy2018,singhPILSTMPhysicsinfusedLong2019,wangIntegratingModeldrivenDatadriven2019,zengTossingBotLearningThrow2019,pitchforth_grey-box_2021}. This can also be seen as residual physics.
    \item $f_\A$ works as an observation function that changes signal's modality \citep{greydanusHamiltonianNeuralNetworks2019,lutterDeepLagrangianNetworks2019,yildizODE2VAEDeepGenerative2019,linialGenerativeODEModeling2020,tothHamiltonianGenerativeNetworks2020,cranmerLagrangianNeuralNetworks2020,saemundssonVariationalIntegratorNetworks2020,jaques_physics-as-inverse-graphics_2020}.
    \item Output of $\operatorname{solve}[\cdots]$ is used as input features of machine learning model $f_\A$ \citep{karpatnePhysicsguidedNeuralNetworks2017,pawarPhysicsGuidedMachine2020,muralidharPhyNetPhysicsGuided2020,zhangMIDPhyNetMemorizedInfusion2021,belbute-peresCombiningDifferentiablePDE2020,pitchforth_grey-box_2021}.
\end{itemize}

In \citep{senguptaEnsemblingGeophysicalModels2020}, $f_\A$ works as the weight of ensemble of physics models, that is, \begin{equation}
    \mathcal{F} \coloneqq \sum_i f_{\A,i}(\bm{z}_{\A,i}) \cdot \operatorname{solve}[\cdots]_i.
\end{equation}

\paragraph{Inverse problems as (V)AE}
The idea of (Bayesian) inverse problems is in line with the auto-encoding variational Bayes; in inverse problems, the forward process (i.e., a decoder) is known and a corresponding backward process (i.e., an encoder) is to be estimated.
For example,
\citet{taitVariationalAutoencodingPDE2020} propose a VAE whose decoder has a structure based on the finite element method for PDEs.
\citet{aragon-calvoSelfsupervisedLearningPhysicsaware2020} replace VAE's decoder with a light distribution model of galaxies for inferring parameters of galaxy from images.
\citet{pakravanSolvingInversePDEProblems2020} integrate a PDE solver into the decoder of a VAE.
\citet{nguyen_model-constrained_2021} discuss the form of solution for a special case where physics and VAEs are with linear models.
\citet{sun_amortized_2021} use learned surrogate models as the decoder of autoencoders.
Similar problems are also discussed in the context of data assimilation \citep[see, e.g.,][]{frerix_variational_2021} and likelihood-free inference \citep[see, e.g.,][]{cranmer_frontier_2020}.

\subsection{Physics-informed methods}

We already introduced some studies in this direction, i.e., designing learning objective based on physics knowledge, in Section~4 of the main text.
We call such an approach \emph{physics-informed} after the work of \citet{raissiPhysicsinformedNeuralNetworks2019}.
As it is not our main interest in this paper, we do not repeat the contents of Section~4; please refer to Section~4, and we also recommend consulting survey papers such as \citep{karniadakis_physics-informed_2021}.
The study by \citet{wang_understanding_2021} is also notable here as they analyze the difficulty of training physics-informed neural networks and propose a remedy.


\subsection{Physics-inspired methods}

While the main interest of this work is integration of \emph{application-specific} physics models into machine learning models, it is worth noting that there are lines of studies where the aim is to design models on the basis of \emph{abstract and general} knowledge of data-generating process.
The extent of the abstraction is diverse; in some studies, it is still natural to refer to the utilized knowledge as physics-related (in a narrow sense, i.e., as one of scientific disciplines) \citep{chung_recurrent_2015,fraccaro_sequential_2016,karl_deep_2017,krishnan_structured_2017,li_disentangled_2018,longPDEnetLearningPDEs2018,longPDENetLearningPDEs2019,yildizODE2VAEDeepGenerative2019,leguenDisentanglingPhysicalDynamics2020,zhang_physics-guided_2020}, and in some other studies, the level of abstraction goes beyond that, e.g., a general model that can realize structural causal models is incorporated \citep{leeb_structured_2021}.
Hence, the heading of this subsection, \emph{physics-inspired}, may not be perfect; we stick to it just for the consistency with the other perspectives.

For example, researchers have been investigating structured generative models for sequential data, in which the structure of latent variables reflects the sequential nature of data \citep{chung_recurrent_2015,fraccaro_sequential_2016,karl_deep_2017,krishnan_structured_2017,li_disentangled_2018}.
Moreover, \citet{casale_gaussian_2018} proposed to place a Gaussian process prior in VAEs.
Note that these studies are never exclusive with the interest of our work and related ones; for example, the VAEs with sequential structures are indeed closely related to the VAEs with ODEs/PDEs \citep[e.g.,][]{yildizODE2VAEDeepGenerative2019,longPDEnetLearningPDEs2018,longPDENetLearningPDEs2019,leguenDisentanglingPhysicalDynamics2020,zhang_physics-guided_2020}, since only the major difference is whether time is discrete or continuous.
The techniques of the structured latent variable models would also be useful in physics-inspired and physics-integrated methods.

\section{Detailed experimental settings}
\label{setting}

\subsection{Infrastructure}

We implemented the models using Python 3.8.0 with PyTorch 1.7.0 and NumPy 1.19.2 throughout the experiments.
We used SciPy of version 1.5.2 in generating the synthetic datasets.
The computation was performed with a machine equipped with an NVIDIA\textsuperscript{\textregistered} Tesla\textsuperscript{\texttrademark} V100 GPU in the experiment on the galaxy images dataset.
We used a machine equipped with a CPU of Intel\textsuperscript{\textregistered} Xeon\textsuperscript{\textregistered} Gold 6148 in the other experiments.


\subsection{Forced damped pendulum}

\paragraph{Data-generating process}
We consider a gravity pendulum with damping effect and external force.
Let $\vartheta(t)$ be the angle of the pendulum at time $t$.
We generated the data by numerically integrating an ODE:
\begin{equation*}
    \frac{\mathrm{d}^2\vartheta(t)}{\mathrm{d}t^2} + \omega^2 \sin\vartheta(t) + \xi \frac{\mathrm{d}\vartheta(t)}{\mathrm{d}t} - A \omega^2 \cos(2 \pi \phi t) = 0,
\end{equation*}
using \texttt{scipy.integrate.solve\_ivp} with the explicit Runge--Kutta method of order 8.
The tolerance parameters \texttt{rtol} and \texttt{atol} were kept to be the default values, $10^{-3}$ and $10^{-6}$, respectively.
We evaluated the solution's values at timesteps $t=0, \Delta t, \cdots, (\tau-1)\Delta t$ with $\Delta t=0.05$ and $\tau=50$ using the 7-th order interpolation polynomial.
The values of the parameters, $\omega$, $\xi$, $A$, and $\phi$, as well as the initial condition $\vartheta(0)$ were randomly sampled when creating each sequence.
The random sampling was with the uniform distributions on the following ranges: $\omega \in [0.785, 3.14]$, $\xi \in [0, 0.8]$ $f \in [3.14, 6.28]$, $A \in [0, 40]$, and $\vartheta(0) \in [-1.57, 1.57]$.
The initial condition of $\dot\vartheta(0)$ was fixed to be $0$.
Each element of each generated sequence was added by zero-mean Gaussian noise with standard deviation $0.01$.

\paragraph{Data property}
The overall dataset we generated comprises 3,500 elements (data-points) in total.
Each data-point $\bm{x}$ is a sequence of length $\tau$ of pendulum's angle, that is,
\begin{equation*}
    \bm{x}_i \coloneqq [ \vartheta_i(0) \ \vartheta_i(\Delta t) \ \cdots \ \vartheta_i((\tau-1)\Delta t) ]^\tr \in \mathbb{R}^\tau,
\end{equation*}
where $i=1,\dots,3500$ is the sample index.

\paragraph{Train/valid/test split}
We first extracted 500 and 1,000 sequences randomly from the overall dataset as the validation set and the test set, respectively.
We then selected 1,000 sequences out of the remaining 2,000 sequences to make a training set.
This selection was randomly done every time; so a different random seed resulted in a different training set.

\paragraph{Physics model}
A part of the data-generating process was given as physics model: $f_\P(\vartheta, z_\P) \coloneqq \ddot{\vartheta}+z_\P\sin\vartheta$.

\paragraph{Latent variables}
By construction of $f_\P$, $z_\P\in\mathbb{R}$ is expected to work in the same manner as $\omega$ in the data-generating process.
There were also $\bm{z}_{\A,1}\in\mathbb{R}$ and $\bm{z}_{\A,2}\in\mathbb{R}^2$ in the full \texttt{NN+phys} and \texttt{NN+phys+reg} models.
Meanwhile, we used $\bm{z}_{\A,2}\in\mathbb{R}^4$ (and no $\bm{z}_{\A,1}$, $\bm{z}_\P$) in the \texttt{NN-only}; and $\bm{z}_{\A,1}\in\mathbb{R}^2$ and $\bm{z}_{\A,2}\in\mathbb{R}^2$ (and no $\bm{z}_\P)$ in the \texttt{NN+solver} model.

\paragraph{Decoder architecture}
We describe the decoder architecture of the full \texttt{NN+phys} and \texttt{NN+phys+reg} models.
In the first stage, an ODE $f_\P(\vartheta,z_\P)+f_{\A,1}(\vartheta,\bm{z}_{\A,1})=0$ is numerically solved with the Euler method for length $\tau$ with step size $\Delta t$.
Let $\bm\nu\in\mathbb{R}^\tau$ be the solution sequence.
In the second stage, $\bm\nu$ is then augmented by $f_{\A,2}$, i.e., $f_{\A,2}(\bm\nu, \bm{z}_{\A,2})$.
We modeled $f_{\A,1}$ with a multilayer perceptron (MLP) with two hidden layers of size 64.
We modeled $f_{\A,2}$ also with an MLP with two hidden layers of size 128.
We used the exponential linear unit (ELU) with its\footnote{\label{alpha}$\alpha$ here is different from one of the hyperparameters of the proposed regularizers.} $\alpha=1.0$ as activation function after the hidden layers.

\paragraph{Encoder architecture}
We describe the encoder architecture of the full \texttt{NN+phys} and \texttt{NN+phys+reg} models.
We modeled the recognition networks, $g_{\A,1}$, $g_{\A,2}$, and $g_{\P,2}$ with MLPs with five hidden layers of size 128, 128, 256, 64, and 32.
We modeled $g_{\P,1}$ as $g_{\P,1}(\bm{x},\bm{z}_{\A,1},\bm{z}_{\A,2})=\bm{x}+U(\bm{x},\bm{z}_{\A,1},\bm{z}_{\A,2})$, where $U$ was an MLP with two hidden layers of size 128.
We used ELU with its\textsuperscript{\ref{alpha}} $\alpha=1.0$ as activation function after the hidden layers.
We put a softplus function after the final output of $g_\P$ to make its output positive-valued.

\paragraph{Replacement functions}
To create the reduced models, we replaced $f_{\A,1}$ and $f_{\A,2}$ respectively by $h_{\A,1}=0$ and $h_{\A,2}=\mathrm{Id}$.

\paragraph{Hyperparameters}
We selected the hyperparameters, $\alpha$, $\beta$, and $\gamma$, from the following sets: $\alpha \in \{ 10^{-3}, 10^{-2}, 10^{-1} \}$, $\beta \in \{ 10^{-4}, 10^{-3}, 10^{-2} \}$, and $\gamma \in \{ 10^{-2}, 10^{-1}, 1 \}$,.
These ranges were chosen to roughly adjust the values of the corresponding regularizers to that of the ELBO.
The configuration that achieved the best reconstruction error on the validation set was selected finally: $\alpha=10^{-2}$, $\beta=10^{-3}$, and $\gamma=10^{-1}$.
In computing $R_{\text{DA},2}$, we sampled $z^*_\P$ from the uniform distribution on range $[0.392, 3.53]$.

\paragraph{Optimization}
We used the Adam optimizer with its\footnote{\label{alphagamma}$\alpha$ and $\gamma$ here are different from the ones of the hyperparameters of the proposed regularizers.} $\alpha=10^{-3}$, $\gamma_1=0.9$, $\gamma_2=0.999$, and $\epsilon=10^{-3}$.
We ran iterations with mini-batch size 200 for 5000 epochs (i.e., 25,000 iterations in total) and saved the model that achieved the best validation reconstruction error.


\subsection{Advection-diffusion system}

\paragraph{Data-generating process}
We consider the advection (convection) and diffusion of something (e.g., heat) on the 1-dimensional space, which is described by the following PDE:
\begin{equation*}
    \frac{\partial T(t,s)}{\partial t} - a \frac{\partial^2 T(t,s)}{\partial s^2} + b \frac{\partial T(t,s)}{\partial s} = 0,
\end{equation*}
where $t$ and $s$ denote the time and space dimension, respectively.
We numerically solved this PDE using \texttt{scipy.integrate.solve\_ivp} with the explicit Runge--Kutta method of order 8.
The spatial derivative was computed with discretization on the $H$-point even grid between $s=0$ and $s=s_\text{max}$ with $H=12$ and $s_\text{max}=2$.
We evaluated the solutions values at timesteps $t=0, \Delta t, \cdots, (\tau-1)\Delta t$ with $\Delta t=0.02$ and $\tau=50$.
The initial condition was set $T(0,s)=c\sin(\pi s / s_\text{max})$, and we set the Dirichlet boundary condition $T(t,0)=T(t,s_\text{max})=0$.
The values of the parameters $a$, $b$, and $c$ were randomly sampled when creating each sequence.
The random sampling was with the uniform distributions on the following ranges: $a \in [10^{-2}, 10^{-1}]$, $b \in [10^{-2}, 10^{-1}]$, and $c \in [0.5, 1.5]$.
Each element of each generated sequence was added by zero-mean Gaussian noise with standard deviation $0.001$.

\paragraph{Data property}
The overall dataset we generated comprises 3,500 sequences, each of which is
\begin{equation*}
    \bm{x}_i \coloneqq
    \begin{bmatrix}
        T_i(0,0) & T_i(\Delta t, 0) & \cdots & T_i((\tau-1)\Delta t, 0) \\
        \vdots & \vdots & & \vdots \\
        T_i(0,s_\text{max}) & T_i(\Delta t, s_\text{max}) & \cdots & T_i((\tau-1)\Delta t, s_\text{max})
    \end{bmatrix}
    \in \mathbb{R}^{H \times \tau}.
\end{equation*}

\paragraph{Train/valid/test split}
We first extracted 500 and 1,000 sequences randomly from the overall dataset as the validation set and the test set, respectively.
We then selected 1,000 sequences out of the remaining 2,000 sequences to make a training set.
This selection was randomly done every time; so a different random seed resulted in a different training set.

\paragraph{Physics model}
A part of the data-generating process was given as physics model: $f_\P(T,\bm{z}_\P) \coloneqq T_t - z_\P T_{ss}$.

\paragraph{Latent variables}
By construction of $f_\P$, $z_\P \in \mathbb{R}$ is expected to work in the same manner as $a$ in the data-generating process.
There was also $\bm{z}_\A\in\mathbb{R}^4$ in the full \texttt{NN+phys} and \texttt{NN+phys+reg} models.
Meanwhile, we used $\bm{z}_\A\in\mathbb{R}^5$ (and no $\bm{z}_\P$) in the \texttt{NN-only} and \texttt{NN+solver} models.

\paragraph{Decoder architecture}
We describe the decoder architecture of the full \texttt{NN+phys} and \texttt{NN+phys+reg} models.
In $\mathcal{F}$, a PDE $f_\P(T,z_\P)+f_\A=0$ was numerically solved with the finite difference method with the explicit scheme for length $\tau$ with temporal step size $\Delta t$.
We modeled $f_\A$ with an MLP with two hidden layers of size 64.
We used ELU with its\textsuperscript{\ref{alpha}} $\alpha=1.0$ as activation function after the hidden layers.
In the \texttt{NN-only} model, we modeled $f_\A$ with an MLP with a hidden layer of size 128.

\paragraph{Encoder architecture}
We describe the encoder architecture of the full \texttt{NN+phys} and \texttt{NN+phys+reg} models.
We modeled the recognition networks, $g_\A$ and $g_{\P,2}$, with MLPs with five hidden layers of size 256, 256, 256, 64, and 32.
We modeled $g_{\P,1}(\bm{x},\bm{z}_\A)$ with an MLP with two hidden layers of size 256.
We used ELU with its\textsuperscript{\ref{alpha}} $\alpha=1.0$ as activation function after the hidden layers.
We put a softplus function after the final output of $g_\P$ to make its output positive-valued.

\paragraph{Replacement functions}
To create the reduced model, we replaced $f_\A$ by $h_\A=0$.

\paragraph{Hyperparameters}
We selected the hyperparameters, $\alpha$, $\beta$, and $\gamma$, from the following sets: $\alpha \in \{ 10^{-2}, 10^{-1} \}$, $\beta \in \{ 10^{-2}, 10^{-1} \}$, and $\gamma \in \{ 10^5, 10^6 \}$.
These ranges were chosen to roughly adjust the values of the corresponding regularizers to that of the ELBO.
The configuration that achieved the best reconstruction error on the validation set was selected finally: $\alpha=10^{-1}$, $\beta=10^{-2}$, and $\gamma=10^6$.
In computing $R_{\text{DA},2}$, we sampled $z^*_\P$ from the uniform distribution on range $[0.005, 0.2]$.

\paragraph{Optimization}
We used the Adam optimizer with its\textsuperscript{\ref{alphagamma}} $\alpha=10^{-3}$, $\gamma_1=0.9$, $\gamma_2=0.999$, and $\epsilon=10^{-3}$.
We ran iterations with mini-batch size 200 for 20000 epochs (i.e., 100,000 iterations in total) and saved the model that achieved the best validation reconstruction error.


\subsection{Galaxy images}

\paragraph{Data property}
We used images of galaxies from a part of the Galaxy10 dataset\footnote{The original images are from the Sloan Digital Sky Survey \url{www.sdss.org}, and the labels are from the Galaxy Zoo project \url{www.galaxyzoo.org}. The dataset is available a part of the \texttt{astroNN} package \citep{leungDeepLearningMultielement2018}}.
We selected the 589 images of the ``Disk, Edge-on, No Bulge'' class to form an overall dataset.
Each image is of size $69 \times 69$ with three channels, so $\bm{x}_i \in \mathbb{R}^{69 \times 69 \times 3}$.
We normalized the intensity values into range $[0,1]$.

\paragraph{Train/valid/test split}
We separated the overall dataset them into training, validation, and test sets with 400, 100, and 89 images, respectively.
In training, we performed data augmentation with random vertical/horizontal flips and random rotation, and thus the size of the training set was 8,000.

\paragraph{Physics model}
The physics model $f_\P \colon \mathbb{R}^{4} \to \mathbb{R}^{69 \times 69}$ is an exponential profile of the light distribution of galaxies whose input is $\bm{z}_\P \coloneqq [I_0 \ A \ B \ \vartheta]^\tr \in \mathbb{R}_{>0}^4$, whose elements have the semantics introduced in the following.
Let $[f_\P]_{i,j}$ denote the $(i,j)$-element of the output of $f_\P$.
Then, for $1 \leq i,j \leq 69$,
\begin{equation*}
    [f_\P]_{i,j} = I_0 \exp(-r_{i,j}),
\end{equation*}
where
\begin{gather*}
    r_{i,j}^2 \coloneqq \frac{(\mathsf{X}_j\cos\vartheta - \mathsf{Y}_i\sin\vartheta)^2}{A^2} + \frac{(\mathsf{X}_j\sin\vartheta + \mathsf{Y}_i\cos\vartheta)^2}{B^2},
    \\
    \mathsf{X}_j \coloneqq 2 \cdot \frac{j-1}{68} -1,
    \\
    \mathsf{Y}_i \coloneqq -2 \cdot \frac{i-1}{68} +1.
\end{gather*}
$(\mathsf{X}_j, \mathsf{Y}_i)$ is the coordinate on the $69 \times 69$ even grid on $[-1,1] \times [-1,1]$.
$I_0$ determines the overall magnitude of the light distribution, $A$ and $B$ determine the size of the ellipse of the light distribution, and $\vartheta$ determines its rotation.
This model was used in a similar problem of \citet{aragon-calvoSelfsupervisedLearningPhysicsaware2020}, where they only handle artificial images.
See also, e.g., \citet{erwinImfitFastFlexible2015}, for an extensive list of such light distribution models of galaxies.

\paragraph{Latent variables}
$\bm{z}_\P\in\mathbb{R}^4$ contains the information of intensity, semi-major and semi-minor axes, and rotation, as mentioned above.
We used $\bm{z}_\A\in\mathbb{R}^2$ in the full \texttt{NN+phys} and \texttt{NN+phys+reg} models.
Meanwhile, we used $\bm{z}_\A\in\mathbb{R}^6$ (and no $\bm{z}_\P$) in the \texttt{NN-only} model.

\paragraph{Decoder architecture}
There is no nontrivial equation-solving process this time because the physics model $f_\P$ itself gives the closed-form solution.
So the data-generating process in the full \texttt{NN+phys} and \texttt{NN+phys+reg} models is:
\begin{equation*}
    \mathcal{F} [f_\P, f_{\A,\text{Unet}}, f_{\A,\text{tconv}}; \bm{z}_\P, \bm{z}_\A]
    \coloneqq f_{\A,\text{Unet}} \big( f_\P(\bm{z}_\P), f_{\A,\text{tconv}}(\bm{z}_\A) \big).
\end{equation*}
$f_{\A,\text{tconv}}$ is a neural net with transposed convolutional layers and given $\bm{z}_\A$, outputs a signal in $\mathbb{R}^{69 \times 69}$.
$f_{\A,\text{Unet}}$ is a neural net with architecture similar to the U-Net, whose outputs are in $\mathbb{R}^{69 \times 69 \times 3}$.
We used the rectified linear unit (ReLU) as activation function and applied batch normalization before each activation function.
In the \texttt{NN-only} model, we modeled $f_\A(\bm{z}_\A)$ only with a neural net with transposed convolutional layers whose output is in $\mathbb{R}^{69 \times 69 \times 3}$.

Note that we do not consider the \texttt{NN+solver} type of baseline as there appear no nontrivial solvers.

\paragraph{Encoder architecture}
The architecture of $g_{\P,2}$ and $g_\A$ is similar to the one in \citet{aragon-calvoSelfsupervisedLearningPhysicsaware2020}.
We put the softplus function after the final output of $g_\P$ to make its output positive-valued.
$g_{\P,1}$ is simply $g_{\P,1}(\bm{x}) \coloneqq \sum_{i=1}^3 c_i[\bm{x}]_i$, where $[\bm{x}]_i$ denotes the $i$-th channel of $\bm{x}$, and $c$'s are trainable parameters.

\paragraph{Replacement functions}
To create the reduced model, we replaced $f_{\A,\text{Unet}}$ by $h_\A$ such that $h_\A(\bm\nu) \coloneqq [\bm\nu; \bm\nu; \bm\nu]\in\mathbb{R}^{69 \times 69 \times 3}$ (i.e., the $\operatorname{repeat}$ operator along the channel axis).

\paragraph{Hyperparameters}
We selected the hyperparameter $\alpha$ from $\alpha \in \{ 10^{-2}, 10^{-1}, 1 \}$.
This range was chosen to roughly adjust the value of the corresponding regularizer to that of the ELBO.
The others were fixed to be $\beta=1$ and $\gamma=10^3$; these values were also determined by roughly adjusting the order of the values of objectives.
In computing $R_{\text{DA},2}$, we sampled from the uniform distributions on $I^*_0 \in [0.5, 1]$, $A^* \in [0.1, 1.0]$, $e^* \in [0.2, 0.8]$, and $\vartheta^* \in [0, 3.142]$, where $B=A(1-e)$.



\subsection{Human gait}

\paragraph{Physics model}
We modeled $f_\P$ with a trainable Hamilton's equation as in \citep{tothHamiltonianGenerativeNetworks2020,greydanusHamiltonianNeuralNetworks2019}:
\begin{equation*}
    f_\P \left( \begin{bmatrix} \bm{p}^\tr & \bm{q}^\tr \end{bmatrix}^\tr, \bm{z}_\P \right) = \begin{bmatrix} -\frac{\partial \mathcal{H}}{\partial \bm{q}}^\tr & \frac{\partial \mathcal{H}}{\partial \bm{p}}^\tr \end{bmatrix}^\tr,
\end{equation*}
where $\bm{p} \in \mathbb{R}^{d_\mathrm{H}}$ is a generalized position, $\bm{q} \in \mathbb{R}^{d_\mathrm{H}}$ is a generalized momentum, and $\mathcal{H} \colon \mathbb{R}^{d_\mathrm{H}} \times \mathbb{R}^{d_\mathrm{H}} \to \mathbb{R}$ is a Hamiltonian.
We let $d_\mathrm{H}=3$ and modeled $\mathcal{H}$ with an MLP with two hidden layers of size.


\paragraph{Latent variables}
$\bm{z}_\P\in\mathbb{R}^{2d_\mathrm{H}}$ is used as the initial condition of $\bm{p}$ and $\bm{q}$.
There was also $\bm{z}_\A\in\mathbb{R}^{15}$.

\paragraph{Decoder architecture}
In the full \texttt{NN+phys} and \texttt{NN+phys+reg} models, the decoding process contains a numerical solver of ODE $f_\P=0$ with the Euler method.
Its output is then transformed by $f_\A$, an MLP with two hidden layers of size 512.

\paragraph{Encoder architecture}
$g_\P$ and $g_\A$ are MLPs with five hidden layers of size 512, 512, 512, 64, 32.

\paragraph{Replacement functions}
To create the reduced model, we replaced $f_\A$ by an affine map $h_\A$, where $h_\A$ is applied to each snapshot of a sequence independently.

\paragraph{Hyperparameters}
We selected the hyperparameter $\alpha$ from $\alpha \in \{ 10^{-3}, 10^{-2}, 10^{-1}, 1 \}$.
This range was chosen to roughly adjust the value of the corresponding regularizer to that of the ELBO.
The other hyperparameters were just $\gamma=\beta=0$ as we did not use the corresponding regularizers.

\section{Additional experimental results}
\label{result}

We present additional experimental results including investigation of the sensitivity of hyperparameter values and some observation on training runtime.


\subsection{Forced damped pendulum}

\begin{figure*}[b]
    \centering
    \pgfplotsset{height=4cm,width=5cm}
    \begin{minipage}[t]{\linewidth}
        \centering
        \begin{tikzpicture}
            \begin{axis}[compat=newest,
                xmode=log,
                label style={font=\scriptsize}, ylabel={MAE of reconstruction}, 
                enlarge x limits=false,
                xmin=0.001, xmax=0.1, ymin=0.2, ymax=1.7,
                xtick pos=left, ytick pos=left,
                legend entries={NN-only, phys-only, NN+solver, NN+phys, NN+phys+reg},
                legend style={at={(0.98,0.8)}, anchor=north east, nodes={scale=0.65, transform shape}, inner sep=1pt},
                legend columns=1,
                legend image code/.code={
                    \draw[mark repeat=3,mark phase=2]
                    plot coordinates {
                        (0cm,0cm)
                        (0.25cm,0cm)        
                        (0.5cm,0cm)         
                    };%
                },
                ]
                \addplot [myorange, dash dot, thick] coordinates {(0.001,0.4380) (0.1,0.4380)}; 
                \addplot [myyellow, dashed, thick] coordinates {(0.001,1.548) (0.1,1.548)}; 
                \addplot [mygreen, densely dotted, very thick] coordinates {(0.001,0.4388) (0.1,0.4388)}; 
                \addplot [myblue, dash dot, very thick] coordinates {(0.001,0.3701) (0.1,0.3701)}; 
                \addplot [myred, mark=o, error bars/.cd, y dir=both, y explicit] table [x=alpha, y=recerr_avg, y error=recerr_std] {pendulum/hyperparam_alpha.txt};
            \end{axis}
        \end{tikzpicture}
        \begin{tikzpicture}
            \begin{axis}[compat=newest,
                xmode=log,
                label style={font=\scriptsize}, 
                enlarge x limits=false,
                xmin=0.0001, xmax=0.01, ymin=0.2, ymax=1.7,
                xtick pos=left, ytick pos=left,
                ]
                \addplot [myyellow, dashed, thick] coordinates {(0.0001,1.548) (0.01,1.548)}; 
                \addplot [myorange, dash dot, thick] coordinates {(0.0001,0.4380) (0.01,0.4380)}; 
                \addplot [mygreen, densely dotted, very thick] coordinates {(0.0001,0.4388) (0.01,0.4388)}; 
                \addplot [myblue, dash dot, very thick] coordinates {(0.0001,0.3701) (0.01,0.3701)}; 
                \addplot [myred, mark=o, error bars/.cd, y dir=both, y explicit] table [x=gamma, y=recerr_avg, y error=recerr_std] {pendulum/hyperparam_gamma.txt};
            \end{axis}
        \end{tikzpicture}
        \begin{tikzpicture}
            \begin{axis}[compat=newest,
                xmode=log,
                label style={font=\scriptsize}, 
                enlarge x limits=false,
                xmin=0.01, xmax=1.0, ymin=0.2, ymax=1.7,
                xtick pos=left, ytick pos=left,
                ]
                \addplot [myyellow, dashed, thick] coordinates {(0.01,1.548) (1.0,1.548)}; 
                \addplot [myorange, dash dot, thick] coordinates {(0.01,0.4380) (1.0,0.4380)}; 
                \addplot [mygreen, densely dotted, very thick] coordinates {(0.01,0.4388) (1.0,0.4388)}; 
                \addplot [myblue, dash dot, very thick] coordinates {(0.01,0.3701) (1.0,0.3701)}; 
                \addplot [myred, mark=o, error bars/.cd, y dir=both, y explicit] table [x=beta, y=recerr_avg, y error=recerr_std] {pendulum/hyperparam_beta.txt};
            \end{axis}
        \end{tikzpicture}
    \end{minipage}
    \\
    \begin{minipage}[t]{\linewidth}
        \centering
        \begin{tikzpicture}
            \begin{axis}[compat=newest,
                xmode=log,
                label style={font=\scriptsize}, xlabel={\textcolor{white}{$\beta$}$\alpha$\textcolor{white}{$\beta$}}, ylabel={MAE of inferred $\omega$},
                enlarge x limits=false,
                xmin=0.001, xmax=0.1, ymin=0.05, ymax=1.15,
                xtick pos=left, ytick pos=left,
                ]
                \addplot [myyellow, dashed, thick] coordinates {(0.001,0.2315) (0.1,0.2315)}; 
                \addplot [myblue, dash dot, very thick] coordinates {(0.001,1.036) (0.1,1.036)}; 
                \addplot [myred, mark=o, error bars/.cd, y dir=both, y explicit] table [x=alpha, y=esterr_avg, y error=esterr_std] {pendulum/hyperparam_alpha.txt};
            \end{axis}
        \end{tikzpicture}
        \begin{tikzpicture}
            \begin{axis}[compat=newest,
                xmode=log,
                label style={font=\scriptsize}, xlabel={\textcolor{white}{$\beta$}$\beta$\textcolor{white}{$\beta$}},
                enlarge x limits=false,
                xmin=0.0001, xmax=0.01, ymin=0.05, ymax=1.15,
                xtick pos=left, ytick pos=left,
                ]
                \addplot [myyellow, dashed, thick] coordinates {(0.0001,0.2315) (0.01,0.2315)}; 
                \addplot [myblue, dash dot, very thick] coordinates {(0.0001,1.036) (0.01,1.036)}; 
                \addplot [myred, mark=o, error bars/.cd, y dir=both, y explicit] table [x=gamma, y=esterr_avg, y error=esterr_std] {pendulum/hyperparam_gamma.txt};
            \end{axis}
        \end{tikzpicture}
        \begin{tikzpicture}
            \begin{axis}[compat=newest,
                xmode=log,
                label style={font=\scriptsize}, xlabel={\textcolor{white}{$\beta$}$\gamma$\textcolor{white}{$\beta$}},
                enlarge x limits=false,
                xmin=0.01, xmax=1.0, ymin=0.05, ymax=1.15,
                xtick pos=left, ytick pos=left,
                ]
                \addplot [myyellow, dashed, thick] coordinates {(0.01,0.2315) (1.0,0.2315)}; 
                \addplot [myblue, dash dot, very thick] coordinates {(0.01,1.036) (1.0,1.036)}; 
                \addplot [myred, mark=o, error bars/.cd, y dir=both, y explicit] table [x=beta, y=esterr_avg, y error=esterr_std] {pendulum/hyperparam_beta.txt};
            \end{axis}
        \end{tikzpicture}
    \end{minipage}
    \vspace*{-4ex}
    \caption{Performances on the pendulum data with one of the hyperparameters ($\alpha$, $\beta$, or $\gamma$) varied around the nominal value, while the others maintained. Averages and SDs over five random trials are reported. Reference values are shown in dashed or dotted lines.}
    \label{fig:pendulum_hp}
\end{figure*}
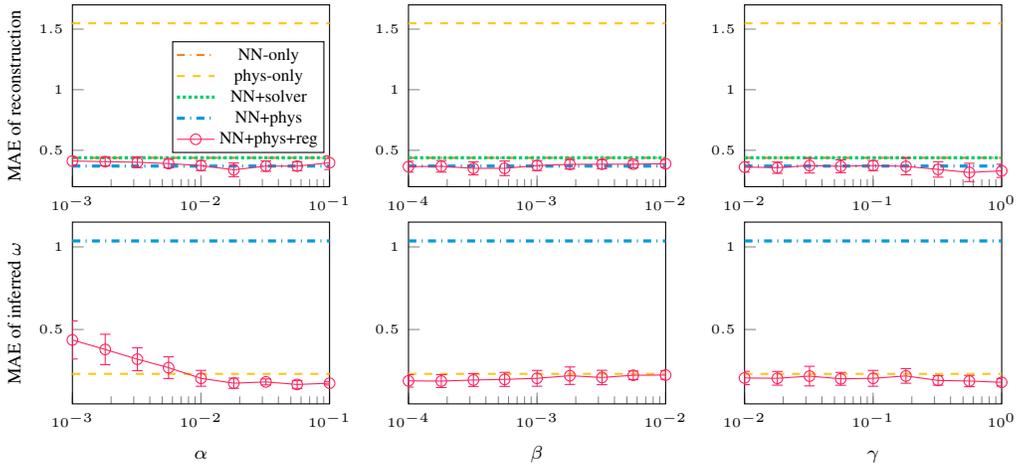

\paragraph{Hyperparameter sensitivity}
We investigated the sensitivity of the performance with regard to the hyperparameters, i.e., the regularization coefficients, $\alpha$, $\beta$, and $\gamma$.
We varied them around the nominal values, i.e., the setting with which the results were reported in the main text ($\alpha=10^{-2}$, $\beta=10^{-3}$, and $\gamma=10^{-1}$; see also Appendix~\ref{setting}).
Figure~\ref{fig:pendulum_hp} summarizes the result.
We can consistently observe the tendency that 1) \texttt{NN+phys+reg} is far better than \texttt{phys-only} in terms of the reconstruction error (upper row); and that 2) \texttt{NN+phys+reg} is far better than \texttt{NN+phys} in terms of the estimation error of physics parameter $\omega$ (lower row).

\paragraph{Achieved hyperparameter values}
We examined the values of the regularizers for data augmentation.
After training, $R_{\text{DA},1} \approx 0.5$ and $R_{\text{DA},2} \approx 2 \times 10^{-3}$ whereas $\Vert x \Vert_2^2 \approx 16$ on average.
This result implies that the functionality of $g_{\P,1}$ and $g_{\P,2}$ are well controlled as intended.

\paragraph{Training runtime}
In training, the \texttt{NN-only} model took about 5.13 seconds for 10 epochs, and the \texttt{NN+phys+reg} took about 10.9 seconds for 10 epochs, though we believe our implementation can still be improved for more efficiency.
The difference probably stems from the physics-part encoder.

\paragraph{More examples of reconstruction and extrapolation}
In the main text, we have shown only one example case of the reconstruction and extrapolation.
In Figure~\ref{fig:pendulum_expl_more}, we provide more examples on different test samples to facilitate further understanding of the result.

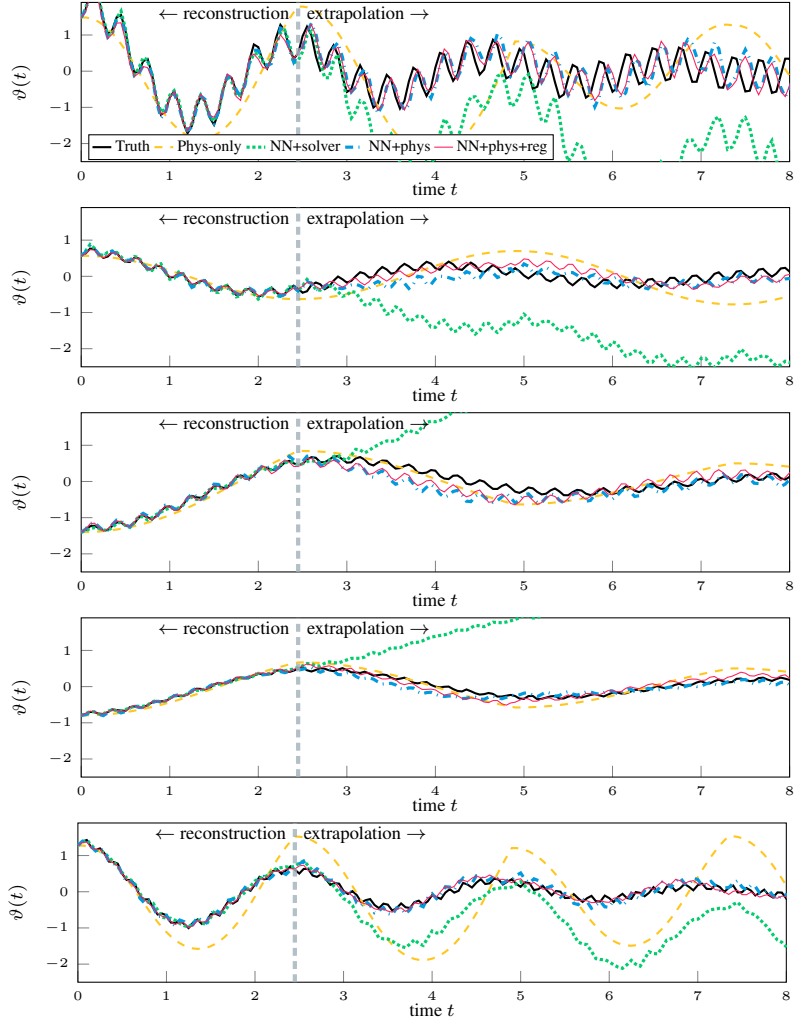
\begin{figure}
    \centering
    \pgfplotsset{height=3.7cm,width=11cm}
\begin{tikzpicture}
    \begin{axis}[compat=newest,
        label style={font=\scriptsize}, xlabel={time $t$}, ylabel={$\vartheta(t)$},
        enlarge x limits=false,
        xmin=0, xmax=8, ymin=-2.5, ymax=1.9,
        xtick pos=left, ytick pos=left,
        legend entries={Truth, Phys-only, NN+solver, NN+phys, NN+phys+reg},
        legend style={at={(0.01,0.01)}, anchor=south west, nodes={scale=0.6, transform shape}, inner sep=1pt},
        legend columns=5,
        legend image code/.code={
            \draw[] 
            plot coordinates {
                (0cm,0cm)
                (0.12cm,0cm)        
                (0.25cm,0cm)         
            };%
        },
        xlabel style={at={(0.5,-1ex)}},
        ]
        \addplot [black, thick] table [x=t, y=data] {pendulum/extrapolation_idx32.txt};
        \addplot [myyellow, dashed, thick] table [x=t, y=noaux] {pendulum/extrapolation_idx32.txt};
        \addplot [mygreen, densely dotted, very thick] table [x=t, y=nophy1] {pendulum/extrapolation_idx32.txt};
        \addplot [myblue, dash dot, very thick] table [x=t, y=phy_noreg] {pendulum/extrapolation_idx32.txt};
        \addplot [myred] table [x=t, y=phy_reg] {pendulum/extrapolation_idx32.txt};
        \addplot [mygray, ultra thick, densely dashed, opacity=0.8, on layer=background] coordinates {(2.45, -2.5) (2.45, 1.9)};
        \node[inner sep=0,outer sep=0,anchor=north west] at (2.55,1.8) {{\scriptsize extrapolation $\rightarrow$}};
        \node[inner sep=0,outer sep=0,anchor=north east] at (2.35,1.8) {{\scriptsize $\leftarrow$ reconstruction}};
    \end{axis}
\end{tikzpicture}
\begin{tikzpicture}
    \begin{axis}[compat=newest,
        label style={font=\scriptsize}, xlabel={time $t$}, ylabel={$\vartheta(t)$},
        enlarge x limits=false,
        xmin=0, xmax=8, ymin=-2.5, ymax=1.9,
        xtick pos=left, ytick pos=left,
        xlabel style={at={(0.5,-1ex)}},
        ]
        \addplot [black, thick] table [x=t, y=data] {pendulum/extrapolation_idx48.txt};
        \addplot [myyellow, dashed, thick] table [x=t, y=noaux] {pendulum/extrapolation_idx48.txt};
        \addplot [mygreen, densely dotted, very thick] table [x=t, y=nophy1] {pendulum/extrapolation_idx48.txt};
        \addplot [myblue, dash dot, very thick] table [x=t, y=phy_noreg] {pendulum/extrapolation_idx48.txt};
        \addplot [myred] table [x=t, y=phy_reg] {pendulum/extrapolation_idx48.txt};
        \addplot [mygray, ultra thick, densely dashed, opacity=0.8, on layer=background] coordinates {(2.45, -2.5) (2.45, 1.9)};
        \node[inner sep=0,outer sep=0,anchor=north west] at (2.55,1.8) {{\scriptsize extrapolation $\rightarrow$}};
        \node[inner sep=0,outer sep=0,anchor=north east] at (2.35,1.8) {{\scriptsize $\leftarrow$ reconstruction}};
    \end{axis}
\end{tikzpicture}
\begin{tikzpicture}
    \begin{axis}[compat=newest,
        label style={font=\scriptsize}, xlabel={time $t$}, ylabel={$\vartheta(t)$},
        enlarge x limits=false,
        xmin=0, xmax=8, ymin=-2.5, ymax=1.9,
        xtick pos=left, ytick pos=left,
        xlabel style={at={(0.5,-1ex)}},
        ]
        \addplot [black, thick] table [x=t, y=data] {pendulum/extrapolation_idx86.txt};
        \addplot [myyellow, dashed, thick] table [x=t, y=noaux] {pendulum/extrapolation_idx86.txt};
        \addplot [mygreen, densely dotted, very thick] table [x=t, y=nophy1] {pendulum/extrapolation_idx86.txt};
        \addplot [myblue, dash dot, very thick] table [x=t, y=phy_noreg] {pendulum/extrapolation_idx86.txt};
        \addplot [myred] table [x=t, y=phy_reg] {pendulum/extrapolation_idx86.txt};
        \addplot [mygray, ultra thick, densely dashed, opacity=0.8, on layer=background] coordinates {(2.45, -2.5) (2.45, 1.9)};
        \node[inner sep=0,outer sep=0,anchor=north west] at (2.55,1.8) {{\scriptsize extrapolation $\rightarrow$}};
        \node[inner sep=0,outer sep=0,anchor=north east] at (2.35,1.8) {{\scriptsize $\leftarrow$ reconstruction}};
    \end{axis}
\end{tikzpicture}
\begin{tikzpicture}
    \begin{axis}[compat=newest,
        label style={font=\scriptsize}, xlabel={time $t$}, ylabel={$\vartheta(t)$},
        enlarge x limits=false,
        xmin=0, xmax=8, ymin=-2.5, ymax=1.9,
        xtick pos=left, ytick pos=left,
        xlabel style={at={(0.5,-1ex)}},
        ]
        \addplot [black, thick] table [x=t, y=data] {pendulum/extrapolation_idx50.txt};
        \addplot [myyellow, dashed, thick] table [x=t, y=noaux] {pendulum/extrapolation_idx50.txt};
        \addplot [mygreen, densely dotted, very thick] table [x=t, y=nophy1] {pendulum/extrapolation_idx50.txt};
        \addplot [myblue, dash dot, very thick] table [x=t, y=phy_noreg] {pendulum/extrapolation_idx50.txt};
        \addplot [myred] table [x=t, y=phy_reg] {pendulum/extrapolation_idx50.txt};
        \addplot [mygray, ultra thick, densely dashed, opacity=0.8, on layer=background] coordinates {(2.45, -2.5) (2.45, 1.9)};
        \node[inner sep=0,outer sep=0,anchor=north west] at (2.55,1.8) {{\scriptsize extrapolation $\rightarrow$}};
        \node[inner sep=0,outer sep=0,anchor=north east] at (2.35,1.8) {{\scriptsize $\leftarrow$ reconstruction}};
    \end{axis}
\end{tikzpicture}
\begin{tikzpicture}
    \begin{axis}[compat=newest,
        label style={font=\scriptsize}, xlabel={time $t$}, ylabel={$\vartheta(t)$},
        enlarge x limits=false,
        xmin=0, xmax=8, ymin=-2.5, ymax=1.9,
        xtick pos=left, ytick pos=left,
        xlabel style={at={(0.5,-1ex)}},
        ]
        \addplot [black, thick] table [x=t, y=data] {pendulum/extrapolation_idx66.txt};
        \addplot [myyellow, dashed, thick] table [x=t, y=noaux] {pendulum/extrapolation_idx66.txt};
        \addplot [mygreen, densely dotted, very thick] table [x=t, y=nophy1] {pendulum/extrapolation_idx66.txt};
        \addplot [myblue, dash dot, very thick] table [x=t, y=phy_noreg] {pendulum/extrapolation_idx66.txt};
        \addplot [myred] table [x=t, y=phy_reg] {pendulum/extrapolation_idx66.txt};
        \addplot [mygray, ultra thick, densely dashed, opacity=0.8, on layer=background] coordinates {(2.45, -2.5) (2.45, 1.9)};
        \node[inner sep=0,outer sep=0,anchor=north west] at (2.55,1.8) {{\scriptsize extrapolation $\rightarrow$}};
        \node[inner sep=0,outer sep=0,anchor=north east] at (2.35,1.8) {{\scriptsize $\leftarrow$ reconstruction}};
    \end{axis}
\end{tikzpicture}
    \caption{Reconstruction and extrapolation of five test samples of the pendulum data. Range $0 \leq t <2.5$ is reconstruction, whereas $t\geq2.5$ is extrapolation. The bottom corresponds to the example presented in the main text.}
    \label{fig:pendulum_expl_more}
\end{figure}


\subsection{Advection-diffusion system}

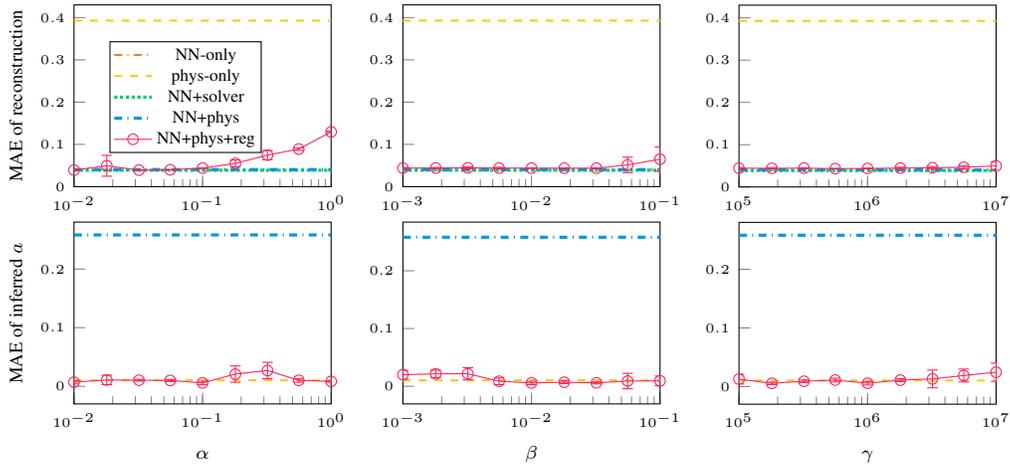
\begin{figure*}[b]
    \centering
    \pgfplotsset{height=4cm,width=5cm}
    \begin{minipage}[t]{\linewidth}
        \centering
        \begin{tikzpicture}
            \begin{axis}[compat=newest,
                xmode=log,
                label style={font=\scriptsize}, ylabel={MAE of reconstruction}, 
                enlarge x limits=false,
                xmin=0.01, xmax=1.0, ymin=0.0, ymax=0.43,
                xtick pos=left, ytick pos=left,
                legend entries={NN-only, phys-only, NN+solver, NN+phys, NN+phys+reg},
                legend style={at={(0.73,0.8)}, anchor=north east, nodes={scale=0.65, transform shape}, inner sep=1pt},
                legend columns=1,
                legend image code/.code={
                    \draw[mark repeat=3,mark phase=2]
                    plot coordinates {
                        (0cm,0cm)
                        (0.25cm,0cm)        
                        (0.5cm,0cm)         
                    };%
                },
                ]
                \addplot [myorange, dash dot, thick] coordinates {(0.01,0.03956) (1.0,0.03956)}; 
                \addplot [myyellow, dashed, thick] coordinates {(0.01,0.3928) (1.0,0.3928)}; 
                \addplot [mygreen, densely dotted, very thick] coordinates {(0.01,0.03882) (1.0,0.03882)}; 
                \addplot [myblue, dash dot, very thick] coordinates {(0.01,0.04043) (1.0,0.04043)}; 
                \addplot [myred, mark=o, error bars/.cd, y dir=both, y explicit] table [x=alpha, y=recerr_avg, y error=recerr_std] {advdif/hyperparam_alpha.txt};
            \end{axis}
        \end{tikzpicture}
        \begin{tikzpicture}
            \begin{axis}[compat=newest,
                xmode=log,
                label style={font=\scriptsize}, 
                enlarge x limits=false,
                xmin=0.001, xmax=0.1, ymin=0.0, ymax=0.43,
                xtick pos=left, ytick pos=left,
                ]
                \addplot [myyellow, dashed, thick] coordinates {(0.001,0.3928) (0.1,0.3928)}; 
                \addplot [myorange, dash dot, thick] coordinates {(0.001,0.03956) (0.1,0.03956)}; 
                \addplot [mygreen, densely dotted, very thick] coordinates {(0.001,0.03882) (0.1,0.03882)}; 
                \addplot [myblue, dash dot, very thick] coordinates {(0.001,0.04043) (0.1,0.04043)}; 
                \addplot [myred, mark=o, error bars/.cd, y dir=both, y explicit] table [x=gamma, y=recerr_avg, y error=recerr_std] {advdif/hyperparam_gamma.txt};
            \end{axis}
        \end{tikzpicture}
        \begin{tikzpicture}
            \begin{axis}[compat=newest,
                xmode=log,
                label style={font=\scriptsize}, 
                enlarge x limits=false,
                xmin=100000, xmax=10000000, ymin=0.0, ymax=0.43,
                xtick pos=left, ytick pos=left,
                ]
                \addplot [myyellow, dashed, thick] coordinates {(100000,0.3928) (10000000,0.3928)}; 
                \addplot [myorange, dash dot, thick] coordinates {(100000,0.03956) (10000000,0.03956)}; 
                \addplot [mygreen, densely dotted, very thick] coordinates {(100000,0.03882) (10000000,0.03882)}; 
                \addplot [myblue, dash dot, very thick] coordinates {(100000,0.04043) (10000000,0.04043)}; 
                \addplot [myred, mark=o, error bars/.cd, y dir=both, y explicit] table [x=beta, y=recerr_avg, y error=recerr_std] {advdif/hyperparam_beta.txt};
            \end{axis}
        \end{tikzpicture}
    \end{minipage}
    \\
    \begin{minipage}[t]{\linewidth}
        \centering
        \begin{tikzpicture}
            \begin{axis}[compat=newest,
                xmode=log,
                label style={font=\scriptsize}, xlabel={\textcolor{white}{$\beta$}$\alpha$\textcolor{white}{$\beta$}}, ylabel={MAE of inferred $a$},
                enlarge x limits=false,
                xmin=0.01, xmax=1.0, ymin=-0.03, ymax=0.28,
                xtick pos=left, ytick pos=left,
                ]
                \addplot [myyellow, dashed, thick] coordinates {(0.01,0.01028) (1.0,0.01028)}; 
                \addplot [myblue, dash dot, very thick] coordinates {(0.01,0.2581) (1.0,0.2581)}; 
                \addplot [myred, mark=o, error bars/.cd, y dir=both, y explicit] table [x=alpha, y=esterr_avg, y error=esterr_std] {advdif/hyperparam_alpha.txt};
            \end{axis}
        \end{tikzpicture}
        \begin{tikzpicture}
            \begin{axis}[compat=newest,
                xmode=log,
                label style={font=\scriptsize}, xlabel={\textcolor{white}{$\beta$}$\beta$\textcolor{white}{$\beta$}},
                enlarge x limits=false,
                xmin=0.001, xmax=0.1,
                xtick pos=left, ytick pos=left,
                ]
                \addplot [myyellow, dashed, thick] coordinates {(0.001,0.01028) (0.1,0.01028)}; 
                \addplot [myblue, dash dot, very thick] coordinates {(0.001,0.2581) (0.1,0.2581)}; 
                \addplot [myred, mark=o, error bars/.cd, y dir=both, y explicit] table [x=gamma, y=esterr_avg, y error=esterr_std] {advdif/hyperparam_gamma.txt};
            \end{axis}
        \end{tikzpicture}
        \begin{tikzpicture}
            \begin{axis}[compat=newest,
                xmode=log,
                label style={font=\scriptsize}, xlabel={\textcolor{white}{$\beta$}$\gamma$\textcolor{white}{$\beta$}},
                enlarge x limits=false,
                xmin=100000, xmax=10000000, ymin=-0.03, ymax=0.28,
                xtick pos=left, ytick pos=left,
                ]
                \addplot [myyellow, dashed, thick] coordinates {(100000,0.01028) (10000000,0.01028)}; 
                \addplot [myblue, dash dot, very thick] coordinates {(100000,0.2581) (10000000,0.2581)}; 
                \addplot [myred, mark=o, error bars/.cd, y dir=both, y explicit] table [x=beta, y=esterr_avg, y error=esterr_std] {advdif/hyperparam_beta.txt};
            \end{axis}
        \end{tikzpicture}
    \end{minipage}
    \vspace*{-4ex}
    \caption{Performances on the advection-diffusion data with one of the hyperparameters ($\alpha$, $\beta$, or $\gamma$) varied around the nominal value, while the others maintained. Averages and SDs over five random trials are reported. Reference values are shown in dashed or dotted lines.}
    \label{fig:advdif_hp}
\end{figure*}

\paragraph{Hyperparameter sensitivity}
We investigated the sensitivity of the performance with regard to the hyperparameters $\alpha$, $\beta$, and $\gamma$.
We varied these values around the nominal values, i.e., the setting with which the results were reported in the main text ($\alpha=10^{-1}$, $\beta=10^{-2}$, and $\gamma=10^6$; see also hyperparameter settings in Appendix~\ref{setting}).
Figure~\ref{fig:advdif_hp} summarizes the result.
Across all the coefficient values, we can consistently observe the tendency similar to that in the pendulum data experiment.

\paragraph{Achieved hyperparameter values}
We examined the values of the regularizers for data augmentation.
After training, $R_{\text{DA},1} \approx 0.01$ and $R_{\text{DA},2} \approx 5 \times 10^{-7}$ whereas $\Vert x \Vert_2^2 \approx 458$ on average.
This result implies that the functionality of $g_{\P,1}$ and $g_{\P,2}$ are well controlled as intended.

\paragraph{Training runtime}
In training, the \texttt{NN-only} model took about 6.01 seconds for 10 epochs, and the \texttt{NN+phys+reg} took about 15.4 seconds for 10 epochs.


\subsection{Galaxy images}

\paragraph{Reconstruction}
In Figure~\ref{fig:galaxy_rec}, we show examples of reconstruction of five test samples.
While the \texttt{phys-only} model cannot recover the color information by construction, the other models that include neural nets reproduce the original colors to some extent.
The reconstruction errors over the whole test set are reported in Table~\ref{tab:galaxy_err}.
From these results, we can observe that the reconstruction performance is similar between \texttt{NN-only}, \texttt{NN+phys}, and \texttt{NN+phys+reg}.
Despite the similar reconstruction performance, the \texttt{NN+phys+reg} model achieves clearly better generation performance as shown in the main text.

\begin{table}[t]
    \begin{minipage}[t]{0.52\textwidth}
        
\caption{Performances on test set of the galaxy image data. Averages (and SDs) over the whole test set are reported.}
\label{tab:galaxy_err}
\setlength{\tabcolsep}{5pt}
\centering
\begin{small}\begin{tabular}{lc}
    \toprule
    & MAE of reconstruction \\
    \midrule
    \texttt{NN-only} & $0.0167$\hspace{1em}($3.0 \times 10^{-2}$) \\
    \texttt{Phys-only} & $0.0264$\hspace{1em}($3.9 \times 10^{-2}$) \\
    \texttt{NN+phys(+reg)}, $\alpha=0$ & $0.0188$\hspace{1em}($3.4 \times 10^{-2}$) \\
    \texttt{NN+phys+reg}, $\alpha>0$ & $0.0180$\hspace{1em}($3.3 \times 10^{-2}$) \\
    \bottomrule
\end{tabular}\end{small}


    \end{minipage}
    \hfill
    \begin{minipage}[t]{0.44\textwidth}
        
\centering
\caption{Performances on test set of the gait data. Averages (SDs) over 20 random trials are reported.}
\label{tab:hm_gait_err}
\setlength{\tabcolsep}{3pt}
\begin{small}\begin{tabular}{m{5.8em}M{3.6em}M{5.2em}}
    \toprule
    & \multicolumn{2}{c}{MAE of reconstruction} \\
    \midrule
    \texttt{Phys-only} &
        $0.726$ & ($1.0 \times 10^{-2}$) \\
    \texttt{NN+solver}&
        $0.276$ & ($1.5 \times 10^{-2}$) \\
    \texttt{NN+phys} &
        $0.273$ & ($9.0 \times 10^{-3}$) \\
    \texttt{NN+phys+reg} &
        $0.259$ & ($9.0 \times 10^{-3}$) \\
    \bottomrule
\end{tabular}\end{small}
    \end{minipage}
\end{table}

\paragraph{Counterfactual generation}
In Figure~\ref{fig:galaxy_cf}, we show the result of generation, where we varied the last element of $\bm{z}_\P$ that corresponds to the angle of a galaxy in image, $\vartheta$.
We examined the models trained without or with one of the regularizers, $R_\text{PPC}$ (i.e., $\alpha=0$); the other regularizers were always active.
In Figure~\ref{fig:galaxy_cf}, the case without the regularizer does not show reasonable generation with different $\vartheta$.
Note that $\vartheta < 0$ was never encountered during training as we set the range of the last element of $\bm{z}_\P$ to be non-negative; nevertheless reasonable images are generated with $\vartheta<0$.

\paragraph{Latent variable}
We computed the first two principal scores of $\bm{z}_\A$ and plotted them with the corresponding image sample in Figure~\ref{fig:galaxy_vis}.
In the \texttt{NN-only} model, the distribution of $\bm{z}_\A$ clearly corresponds to the angle of the galaxy in images\footnote{This might be a good property in some applications, but we do not want for it to happen in our \texttt{NN+phys+reg} model because the angle is rather manually encoded in an element of $\bm{z}_\P$, and $\bm{z}_\A$ should carry other information.}.
In contrast in the \texttt{NN+phys+reg} model, such a correspondence is not observed.
This is a reasonable result because in \texttt{NN+phys+reg}, the semantic of galaxy angle is completely assigned to the last element of $\bm{z}_\P$.

\begin{figure}[p]
    \centering
    \input{fig_galaxy_rec}
    \\[5ex]
    \input{fig_galaxy_cf}
    \\[5ex]
    \input{fig_galaxy_vis}
\end{figure}


\subsection{Human gait}

\paragraph{Reconstruction}
The reconstruction errors over the whole test set are reported in Table~\ref{tab:hm_gait_err}.

\section{Extension}
\label{extension}

While the proposed framework is useful as shown in our experiments, there are several directions to go for possible technical improvement of the method.
First, physics-integrated VAEs can be further combined with techniques to solve ODEs and PDEs with neural networks \citep{raissiPhysicsinformedNeuralNetworks2019,yangPhysicsinformedDeepGenerative2018,yangBPINNsBayesianPhysicsinformed2021}.
We supposed the use of differentiable numerical solvers if the model contains ODEs or PDEs, but such numerical solvers are often computationally heavy.
Replacing them with neural net-based solutions will be useful for various applications.
Second, while we defined the regularizer based on the (possibly loose) upper bound of KL divergence, we may use other dissimilarity measure of distributions or random variables, such as maximum mean discrepancy.
Third, the proposed regularization method can be extended to other types of deep generative models; e.g., an extension to InfoVAE \citep{zhaoInfoVAEBalancingLearning2019} is straightforward.
Lastly, neural architecture search in the context of physics-integrated models \citep{baBlendingDiversePhysical2019} would be an interesting topic also in generative modeling.

\end{document}